\documentclass[]{article}
\usepackage{amssymb,amsmath,amsthm,multirow,graphicx}
\usepackage[paper=a4paper,textwidth=160mm]{geometry}

\usepackage[]{hyperref}
\hypersetup{pdftitle={DEGMC: Denoising Diffusion Models Based on Riemannian Equivariant Group Morphological Convolutions},
colorlinks=true,citecolor=blue,linkcolor=red}

\usepackage{algorithmic,algorithm}

\usepackage{dsfont}

\def\Z{\ensuremath{\mathds{Z}}}
\def\R{\ensuremath{\mathds{R}}}
\def\E{\ensuremath{\mathds{E}}}

\def\B{\ensuremath{\mathds{B}}}

\newcommand{\ud}{\mathrm{d}}

\newcommand{\gM}{\mathcal{M}}
\newcommand{\gN}{\mathcal{N}}

\newcommand{\ie}{\textit{i.e.},}

\newcommand{\resp}{\textit{resp.}}

\newtheorem{theorem}{Theorem}[section]

\newtheorem{definition}{Definition}[section]

\newtheorem{proposition}[theorem]{Proposition}

\usepackage{caption}
\usepackage{subcaption}
\usepackage{tikz}
\usepackage{pgfplots}
\usetikzlibrary{spy,calc}

\usepackage{multirow}
\usepackage{ctable}

\newif\ifblackandwhitecycle
\gdef\patternnumber{0}

\pgfkeys{/tikz/.cd,
    zoombox paths/.style={
        draw=orange,
        very thick
    },
    black and white/.is choice,
    black and white/.default=static,
    black and white/static/.style={ 
        draw=white,   
        zoombox paths/.append style={
            draw=white,
            postaction={
                draw=black,
                loosely dashed
            }
        }
    },
    black and white/static/.code={
        \gdef\patternnumber{1}
    },
    black and white/cycle/.code={
        \blackandwhitecycletrue
        \gdef\patternnumber{1}
    },
    black and white pattern/.is choice,
    black and white pattern/0/.style={},
    black and white pattern/1/.style={    
            draw=white,
            postaction={
                draw=black,
                dash pattern=on 2pt off 2pt
            }
    },
    black and white pattern/2/.style={    
            draw=white,
            postaction={
                draw=black,
                dash pattern=on 4pt off 4pt
            }
    },
    black and white pattern/3/.style={    
            draw=white,
            postaction={
                draw=black,
                dash pattern=on 4pt off 4pt on 1pt off 4pt
            }
    },
    black and white pattern/4/.style={    
            draw=white,
            postaction={
                draw=black,
                dash pattern=on 4pt off 2pt on 2 pt off 2pt on 2 pt off 2pt
            }
    },
    zoomboxarray inner gap/.initial=5pt,
    zoomboxarray columns/.initial=2,
    zoomboxarray rows/.initial=2,
    subfigurename/.initial={},
    figurename/.initial={zoombox},
    zoomboxarray/.style={
        execute at begin picture={
            \begin{scope}[
                spy using outlines={%
                    zoombox paths,
                    width=\imagewidth / \pgfkeysvalueof{/tikz/zoomboxarray columns} - (\pgfkeysvalueof{/tikz/zoomboxarray columns} - 1) / \pgfkeysvalueof{/tikz/zoomboxarray columns} * \pgfkeysvalueof{/tikz/zoomboxarray inner gap} -\pgflinewidth,
                    height=\imageheight / \pgfkeysvalueof{/tikz/zoomboxarray rows} - (\pgfkeysvalueof{/tikz/zoomboxarray rows} - 1) / \pgfkeysvalueof{/tikz/zoomboxarray rows} * \pgfkeysvalueof{/tikz/zoomboxarray inner gap}-\pgflinewidth,
                    magnification=3,
                    every spy on node/.style={
                        zoombox paths
                    },
                    every spy in node/.style={
                        zoombox paths
                    }
                }
            ]
        },
        execute at end picture={
            \end{scope}
            \node at (image.north) [anchor=north,inner sep=0pt] {\subcaptionbox{\label{\pgfkeysvalueof{/tikz/figurename}-image}}{\phantomimage}};
            \node at (zoomboxes container.north) [anchor=north,inner sep=0pt] {\subcaptionbox{\label{\pgfkeysvalueof{/tikz/figurename}-zoom}}{\phantomimage}};
     \gdef\patternnumber{0}
        },
        spymargin/.initial=0.5em,
        zoomboxes xshift/.initial=1,
        zoomboxes right/.code=\pgfkeys{/tikz/zoomboxes xshift=1},
        zoomboxes left/.code=\pgfkeys{/tikz/zoomboxes xshift=-1},
        zoomboxes yshift/.initial=0,
        zoomboxes above/.code={
            \pgfkeys{/tikz/zoomboxes yshift=1},
            \pgfkeys{/tikz/zoomboxes xshift=0}
        },
        zoomboxes below/.code={
            \pgfkeys{/tikz/zoomboxes yshift=-1},
            \pgfkeys{/tikz/zoomboxes xshift=0}
        },
        caption margin/.initial=4ex,
    },
    adjust caption spacing/.code={},
    image container/.style={
        inner sep=0pt,
        at=(image.north),
        anchor=north,
        adjust caption spacing
    },
    zoomboxes container/.style={
        inner sep=0pt,
        at=(image.north),
        anchor=north,
        name=zoomboxes container,
        xshift=\pgfkeysvalueof{/tikz/zoomboxes xshift}*(\imagewidth+\pgfkeysvalueof{/tikz/spymargin}),
        yshift=\pgfkeysvalueof{/tikz/zoomboxes yshift}*(\imageheight+\pgfkeysvalueof{/tikz/spymargin}+\pgfkeysvalueof{/tikz/caption margin}),
        adjust caption spacing
    },
    calculate dimensions/.code={
        \pgfpointdiff{\pgfpointanchor{image}{south west} }{\pgfpointanchor{image}{north east} }
        \pgfgetlastxy{\imagewidth}{\imageheight}
        \global\let\imagewidth=\imagewidth
        \global\let\imageheight=\imageheight
        \gdef\columncount{1}
        \gdef\rowcount{1}
        
    },
    image node/.style={
        inner sep=0pt,
        name=image,
        anchor=south west,
        append after command={
            [calculate dimensions]
            node [image container,subfigurename=\pgfkeysvalueof{/tikz/figurename}-image] {\phantomimage}
            node [zoomboxes container,subfigurename=\pgfkeysvalueof{/tikz/figurename}-zoom] {\phantomimage}
        }
    },
    color code/.style={
        zoombox paths/.append style={draw=#1}
    },
    connect zoomboxes/.style={
    spy connection path={\draw[draw=none,zoombox paths] (tikzspyonnode) -- (tikzspyinnode);}
    },
    help grid code/.code={
        \begin{scope}[
                x={(image.south east)},
                y={(image.north west)},
                font=\footnotesize,
                help lines,
                overlay
            ]
            \foreach \x in {0,1,...,9} { 
                \draw(\x/10,0) -- (\x/10,1);
                \node [anchor=north] at (\x/10,0) {0.\x};
            }
            \foreach \y in {0,1,...,9} {
                \draw(0,\y/10) -- (1,\y/10);                        \node [anchor=east] at (0,\y/10) {0.\y};
            }
        \end{scope}    
    },
    help grid/.style={
        append after command={
            [help grid code]
        }
    },
}

\newcommand\phantomimage{%
    \phantom{%
        \rule{\imagewidth}{\imageheight}%
    }%
}
\newcommand\zoombox[2][]{
    \begin{scope}[zoombox paths]
        \pgfmathsetmacro\xpos{
            (\columncount-1)*(\imagewidth / \pgfkeysvalueof{/tikz/zoomboxarray columns} + \pgfkeysvalueof{/tikz/zoomboxarray inner gap} / \pgfkeysvalueof{/tikz/zoomboxarray columns} ) + \pgflinewidth
        }
        \pgfmathsetmacro\ypos{
            (\rowcount-1)*( \imageheight / \pgfkeysvalueof{/tikz/zoomboxarray rows} + \pgfkeysvalueof{/tikz/zoomboxarray inner gap} / \pgfkeysvalueof{/tikz/zoomboxarray rows} ) + 0.5*\pgflinewidth
        }
        \edef\dospy{\noexpand\spy [
            #1,
            zoombox paths/.append style={
                black and white pattern=\patternnumber
            },
            every spy on node/.append style={#1},
            x=\imagewidth,
            y=\imageheight
        ] on (#2) in node [anchor=north west] at ($(zoomboxes container.north west)+(\xpos pt,-\ypos pt)$);}
        \dospy
        \pgfmathtruncatemacro\pgfmathresult{ifthenelse(\columncount==\pgfkeysvalueof{/tikz/zoomboxarray columns},\rowcount+1,\rowcount)}
        \global\let\rowcount=\pgfmathresult
        \pgfmathtruncatemacro\pgfmathresult{ifthenelse(\columncount==\pgfkeysvalueof{/tikz/zoomboxarray columns},1,\columncount+1)}
        \global\let\columncount=\pgfmathresult
        \ifblackandwhitecycle
            \pgfmathtruncatemacro{\newpatternnumber}{\patternnumber+1}
            \global\edef\patternnumber{\newpatternnumber}
        \fi
    \end{scope}
}

\begin{document}

\title{DEGMC: Denoising Diffusion Models Based on Riemannian Equivariant Group Morphological Convolutions}
\author{
El Hadji S. Diop$^\star$, Thierno Fall$^\star$ and Mohamed Daoudi$^\ddag$\vspace{0.5cm}\\
{\small $^\star$NAGIP-Nonlinear Analysis and Geometric Information Processing Research Group}\\ {\small Department of Mathematics, University Iba Der Thiam, BP $967$, Thies, Senegal} \vspace{0.5cm}\\ 
 {\small $^\ddag$Institut Mines-Telecom Nord Europe, Centre for Digital Systems, CNRS Centrale Lille} \\{\small UMR 9189 CRIStAL, University of Lille, F-$59000$ Lille}
}	
\date{}

\maketitle

%%%%--------------------------------------------------------------------------------
\begin{abstract}
%Diffusion models have recently emerged and have shown great capabilities in high-quality image synthesis and data generation. 
In this work, we address two major issues in recent Denoising Diffusion Probabilistic Models (DDPM): {\bf 1)} geometric key feature extraction and {\bf 2)} network equivariance. Since the DDPM prediction network relies on the U-net architecture, which is theoretically only translation equivariant, we introduce a geometric approach combined with an equivariance property of the more general Euclidean group, which includes rotations, reflections, and permutations. We introduce the notion of group morphological convolutions in Riemannian manifolds, which are derived from the viscosity solutions of first-order Hamilton-Jacobi-type partial differential equations (PDEs) that act as morphological multiscale dilations and erosions. We add a convection term to the model and solve it using the method of characteristics. This helps us better capture nonlinearities, represent thin geometric structures, and incorporate symmetries into the learning process. Experimental results on the MNIST, RotoMNIST, and CIFAR-10 datasets show noticeable improvements compared to the baseline DDPM model.\\ 
 

\end{abstract}

%%%%%%%%%%%%%%%%%%%%%%%%%%%%%%%%%%%%%%%%%%%%%%%%%
\section{Introduction} %%%%%%%%%%%%%%%%%%%%%%%%%%
%%%%%%%%%%%%%%%%%%%%%%%%%%%%%%%%%%%%%%%%%%%%%%%%

Over the past few years, deep generative models have experienced rapid growth, with applications ranging from realistic image generation \cite{Goodfellow2014,Kingma2013,Kingma2014,dhariwal2021diffusion,ho2020denoising} to audio synthesis \cite{chen2020wavegrad,popov2021grad}, and even molecular modeling \cite{simonovsky2018graphvae,gebauer2019symmetry,simm2020symmetry,hoogeboom2022equivariant}. Among these approaches, probabilistic diffusion models (PDM) \cite{SohlDickstein2015,Song2019,ho2020denoising,Song2021,Croitoru2023} have emerged as particularly influential due to their impressive generative capabilities. PDM can be broadly categorized into three main classes: denoising diffusion probabilistic models (DDPM) \cite{SohlDickstein2015,ho2020denoising}, inspired by nonequilibrium thermodynamics; noise-conditioned score networks \cite{Song2019}, based on a multiscale score-matching objective; and stochastic differential equation–based models \cite{Song2021,Huang2021}. 

In particular, in the field of image generation, DDPM \cite{ho2020denoising} have demonstrated a remarkable ability to produce high-quality samples. Their principle relies on two complementary stages. The first stage, known as the forward diffusion process, progressively adds Gaussian noise to the data until its distribution approaches an isotropic normal distribution. The second stage, the reverse or denoising process, aims to invert this procedure by learning, via a deep neural network, the noise to be removed in order to reconstruct the original data. Training is performed within a probabilistic framework by optimizing a variational lower bound (ELBO) on the likelihood, thereby ensuring the theoretical soundness of the model. Compared to other families of generative models, such as variational autoencoders \cite{Kingma2013,Rezende2014} or generative adversarial networks (GAN) \cite{Goodfellow2014,Goodfellow2017}, PDM are distinguished by the stability of their training and the diversity of the generated samples.

Equivariance plays an important role in most neural network architectures. It means that applying a transformation to the input data and then passing it through the network is equivalent to first passing the data through the network and then transforming the output. This property enables the model to learn the symmetries present in the data. This principle has recently been exploited in molecular generation by combining $\mathrm{E}(n)$-equivariant graph neural networks \cite{Satorras2021} with $\mathrm{E}(3)$-equivariance of the denoising distribution in the diffusion process of DDPM \cite{hoogeboom2022equivariant}. A similar approach has been proposed for 3D molecular generation \cite{Cornet2024} with a learnable forward process. An equivariant diffusion model has also been introduced in \cite{Brehmer2023}, exhibiting $\mathrm{SE}(3)\times\Z\times \mathrm{S}_n$ invariance of the trajectories. An $\mathrm{E}(3)$-equivariant model with $\mathrm{O}(3)$ invariance in the conditional diffusion process has been designed in \cite{Igashov2024} for molecular linker design.

Deep neural networks are naturally invariant to translations. To extend this invariance to other types of transformations, group convolutions neural networks (G-CNN) have been introduced \cite{Cohen2016,Bekkers2018,Cohen2019}, generalizing CNN to incorporate symmetries during learning. G-CNN have shown significant improvements over traditional CNN \cite{Winkels2018,Cohen2018,Bekkers2019}. Recently, a PDE-based framework referred to as PDE-G-CNN has been introduced \cite{Smets13July2022,Bellaard2023} as a generalization of G-CNN. In \cite{Diop2024b}, equivariant PDE-G-CNN were integrated into GAN models, demonstrating substantial improvements in sample quality and increased robustness to geometric transformations of the data.

PDM have been extended to Riemannian manifolds through a continuous-time Riemannian ELBO \cite{Huang2022}. A Riemannian extension of DDPM has recently been proposed to learn distributions supported on submanifolds of $\R^n$ \cite{Liu2025}. Score-matching models have also been generalized to Riemannian manifolds \cite{Bortoli2022}. In addition, a generalized strategy for the numerical computation of the heat kernel on Riemannian symmetric spaces within the denoising score-matching framework has been proposed in \cite{Lou2023}. These approaches primarily aim to define diffusion processes that are consistent with Riemannian geometry or to learn distributions on submanifolds of $\R^n$. However, they generally rely on standard geometric architectures and do not explicitly investigate the impact of equivariant layers nor the integration of morphological operators within the denoising networks of PDM. The main objective of this work is to analyze the impact of equivariant operators on noise prediction, symmetry preservation, and the extraction of fine geometric structures in DDPM.

\paragraph{Contributions}  
DDPM use a U-Net architecture, which theoretically satisfies only the property of translation equivariance, to predict the residual noise at each step of the diffusion process. In this work, we introduce the equivariance property via the action of a Lie group. This allows us to exploit the symmetry group structure and the underlying Riemannian geometry simultaneously. We treat each layer of the neural network as a functional operator and model the feature maps of traditional neural networks as functions. We generalize the equivariance property with respect to the more general Euclidean group of transformations, including rotations, reflections, and permutations in $\R^n$. We also ensure that these operators commute with the group actions defined on their respective function spaces. We introduce a denoising diffusion model defined on general Riemannian manifolds. In this model, the noise prediction network is constructed from equivariant layers based on morphological group-equivariant convolutions and convection operators, which replace standard layers. We summarize our contributions as follows:
\begin{itemize}
\item Construction of a network capable of better capturing fine structures and nonlinearities using convection-dilation-erosion blocks (ResnetCDEBlocks).%\textbf{Manifold-based PDE-G-CNN Architecture:} We construct a network capable of better capturing fine structures and nonlinearities using convection-dilation-erosion blocks (ResnetCDEBlocks).
    \item Proposition of a diffusion framework that integrates group morphological operators. %\textbf{Equivariant ED2RM Diffusion Model:} We propose a diffusion framework integrating morphological PDE operators into the noise prediction network.
    \item Systematic preservation of translation, rotation, reflection, and permutation symmetries, and reduction of sampling complexity. %\textbf{Intrinsic Equivariance:} Our model ensures systematic preservation of translation, rotation, reflection, and permutation symmetries, reducing sampling complexity.   
    \item Faster convergence and superior FID scores compared to baseline DDPM due to the introduction of geometric equivariance. %\textbf{Robustness and Performance:} Experiments on MNIST and RotoMNIST datasets demonstrate faster convergence and superior FID scores compared to baseline DDPM, confirming the efficacy of incorporating geometric equivariance.
\end{itemize}
%\begin{itemize}
% \item Introduction of an equivariant denoising diffusion model (ED2RM) integrating morphological PDE operators into the noise prediction network (see Fig.~\ref{fig:archi}).
% \item Design of a diffusion model intrinsically equivariant to translations, rotations, reflections, and permutations, ensuring the preservation of data symmetries throughout training.
% \item Construction of a prediction network based on PDE-G-CNNs defined on Riemannian manifolds, enabling improved capture of nonlinearities and fine geometric structures.
% \item Geometric interpretability of the proposed ED2RM framework, with model operations corresponding to well-known PDEs such as convection equations and Hamilton--Jacobi equations for multiscale dilations and erosions.
% \item Enhanced extraction of key geometric features during the denoising stage, strengthening the overall robustness of the model.
% \item Improved image generation quality, demonstrated through experiments on MNIST and RotoMNIST, where ED2RM exhibits faster convergence and better FID scores than baseline DDPM.
%\end{itemize}

\begin{figure}[!ht]
%\vskip 0.2in
\centering
\includegraphics[width=\linewidth]{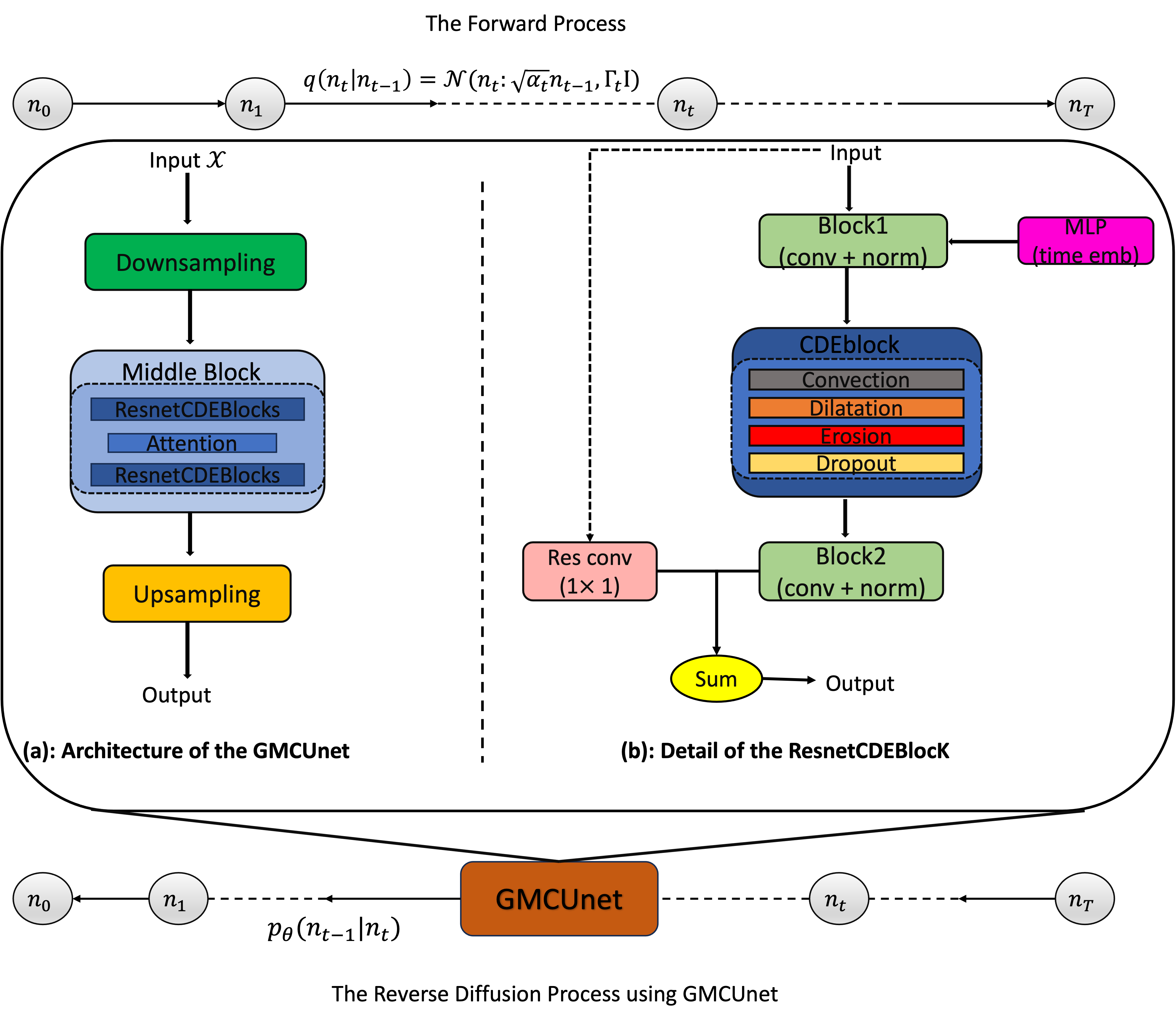}
\caption{Our DEGMC approach uses equivariant group morphological convolution layers integrated into the denoising network referred to as GMCUnet (see Section~\ref{sec:archigmcunet}). This network replaces the standard U-Net architecture used in classical DDPMs during the reverse denoising process. GMCUnet relies on convection, dilation, and erosion (CDE) operations to enforce equivariance with respect to translations, rotations, reflections, and permutations. This improves the extraction and preservation of fine geometric structures throughout the generation process.}
\label{fig:archi}
\end{figure}

%\section{Related Work}

%------------------------------------------------------------------------------
\section{Preliminaries}
In this section, we provide the necessary background on diffusion models, as well as the key concepts from Riemannian geometry and equivariant geometry that we utilize to develop the DEGMC framework. 
%---------------------------------
\subsection{Probabilistic Diffusion Models} 
\label{sec:pdm}

Diffusion models are a class of generative models that operate by progressively corrupting real data with random noise and subsequently learning to reverse this corruption. This denoising process is performed by a neural network that recovers the original data distribution from pure noise. For a more comprehensive derivation, refer to Appendix \ref{annexe:diffusion}.

\paragraph{Forward Process}
The forward (diffusion) process transforms data $n_0$ into a sequence of increasingly noisy latent variables $\{n_1, \dots, n_T\}$. Each step $t$ is modeled as a Markov chain where noise is added via a conditional Gaussian distribution:
\begin{equation*}
    q(n_t \mid n_{t-1}) = \mathcal{N}(n_t ; \sqrt{\alpha_t} n_{t-1}, (1 - \alpha_t) I) \label{eq:eq10}
\end{equation*}
Following the variance-preserving schedule established by \cite{ho2020denoising}, we define the noise level as $\Gamma_t = 1 - \alpha_t$. Here, $\alpha_t \in (0,1)$ dictates the signal-to-noise ratio at each step; as $t$ increases, the mean $\sqrt{\alpha_t} n_{t-1}$ decays, and the signal progressively vanishes into a standard normal distribution.

\paragraph{Inverse Generative Process}
The generative process reverses the forward diffusion by sampling from the true posterior $q(n_{t-1} \mid n_t)$. While this true distribution is intractable, it can be approximated as a Gaussian for sufficiently small noise steps. We utilize a neural network $\phi$, parameterized by $\theta$, to learn the transition $p_{\theta}(n_{t-1} \mid n_t)$:
\begin{equation*}
    p_{\theta}(n_{t-1} \mid n_t) = \mathcal{N}(n_{t-1} ; \mu_{\theta}(n_t, t), \Sigma_{\theta}(n_t, t)) \label{eq:dgeneratif}
\end{equation*}
where $\mu_{\theta}$ and $\Sigma_{\theta}$ represent the predicted mean and covariance at iteration $t$. Following \cite{ho2020denoising}, the covariance is often simplified to a non-learnable time-dependent constant, $\Sigma_{\theta}(n_t, t) = \sigma_{t}^{2} I$.

\paragraph{Variational Lower Bound of the Likelihood} 
The optimization problem for the Evidence Lower Bound (ELBO) with respect to $\theta$ is given by:
\begin{align}
\underset{\theta}{\operatorname{minimize}} \;\;
& \sum_{t=2}^{T}
\operatorname{KL}\!\left(
q\left(n_{t-1} \mid n_{t}, n_{0}\right)
\,\|\, 
p_{\theta}\left(n_{t-1} \mid n_{t}\right)
\right) \notag \\
& \quad
- \E_{q\left(n_{1:T} \mid n_{0}\right)}
\!\left[
\log p_{\theta}\left(n_{0} \mid n_{1}\right)
\right].
\label{eq:elbo}
\end{align}
This equation trains the inverse distribution $p_{\theta}\!\left(n_{t-1} \mid n_{t}\right)$ to match the true denoising distribution $q\!\left(n_{t-1} \mid n_{t}, n_{0}\right)$ by minimizing their KL divergence. It can therefore be used as a loss function for a neural network parameterized by $\theta$, emphasizing the alignment between these two distributions. Details leading to \eqref{eq:elbo} are provided in Appendix~\ref{annexe:elbo}.

\subsection{U-net architecture in DDPM} 
%%%%%%%%%%%%%%%%%%%%%%%%%%%%%%%%%%%%%%%%

The denoising function $\phi$, parameterized by $\theta$, typically uses a \textbf{U-Net} architecture to predict the residual noise at each step of the diffusion process. Its hierarchical structure enables multiscale feature extraction through a symmetric encoder-decoder scheme. While a standard U-Net is theoretically \textbf{translation equivariant}—meaning a spatial shift of the input noise $n_t$ induces a corresponding shift in the output—it does not exhibit intrinsic equivariance with respect to the more general Euclidean group $E(n)$. This group includes rotations, reflections, and permutations in $\R^n$. 
%In this work, we consider each layer of the neural network as a functional operator and model the feature maps of conventional neural networks as functions. We introduce the notion of equivariance with respect to additional group transformations to ensure that these operators commute with the group actions defined on their respective function spaces. This property is typically achieved by replacing standard layers with the introduced morphological group equivariant convolutions.

\subsection{Equivariance}
%%%%%%%%%%%%%%%%%%%%%%%%%%

%In this section, we consider, in general, a connected Riemannian manifold $(\gM,\mathbf{g})$.
%\begin{definition}  
%Let $G$ be a connected Lie group with identity element $e$ and $(\gM,\mathbf{g})$ a connected Riemannian manifold $\gM$ with metric $g$. A left action of $G$ on $(\gM,\mathbf{g})$ is a map $\varphi : G \times (\gM,\mathbf{g}) \to (\gM,\mathbf{g})$ satisfying:  
%\begin{enumerate}  
%\item $\varphi(e,x) = x$, $\forall~x \in (\gM,\mathbf{g})$.  
%\item $\varphi(g, \varphi(h,x)) = \varphi(gh, x)$, $\forall~g,h \in G$ and $\forall~x \in (\gM,\mathbf{g})$.  
%\end{enumerate}  
%\end{definition}  

Let $G$ be a connected Lie group with identity element $e$ and $(\gM,\mathbf{g})$ a connected Riemannian manifold $\gM$ with metric $g$.

Let $\varphi: G \times (\gM,\mathbf{g}) \to (\gM,\mathbf{g})$ be a left action of $G$ on $(\gM,\mathbf{g})$. For a fixed $g \in G$, we define $\varphi_g$:$$\varphi_g: (\gM,\mathbf{g}) \to (\gM,\mathbf{g}), x \mapsto \varphi_g(x) = \varphi(g,x).$$

The map $\varphi: G \times (\gM,\mathbf{g}) \to (\gM,\mathbf{g})$ is a left action if, for all $g,h \in G$, we have: $$\varphi_e = id_M \text{ and } \varphi_g \circ \varphi_h = \varphi_{gh}.$$

Let $\varphi_h : (\gM,\mathbf{g}) \longrightarrow (\gM,\mathbf{g})$ denote the left group action (considered here as a multiplication) by an element \(h \in G\), defined for every \(x \in (\gM,\mathbf{g})\) by:  $$\varphi_h(x) = h \cdot x.$$  
Let \(\mathcal{L}_h\) denote the left regular representation of \(G\) on functions \(f\) defined on \(\gM\), given by: $$(\mathcal{L}_h f)(x) = f(\varphi_{h^{-1}}(x)),$$ where $h^{-1}$ is the inverse of $h \in G$.

Let \(x_0\) be an arbitrary fixed point on the connected Riemannian manifold $(\gM,\mathbf{g})$. Let $\pi : G \rightarrow  (\gM,\mathbf{g})$ denote the projection defined by associating to each element \(h\) of \(G\) a point in $(\gM,\mathbf{g})$ as follows: $$\forall~h \in G, \pi(h) = \varphi_h(x_0).$$
In other words, once a reference point $x_0 \in (\gM,\mathbf{g})$ is chosen, the projection \(\pi(h)\) associates to each element \(h\) of \(G\) the unique point in $(\gM,\mathbf{g})$ to which \(h\) sends \(x_0\) under the action \(\varphi_h\).

Let us consider a connected Lie group \(G\) acting transitively on the connected Riemannian manifold $(\gM,\mathbf{g})$. This means that for any points \(x, y \in (\gM,\mathbf{g})\), there exists an element \(h \in G\) such that $\varphi_h(x) = y$, which corresponds to the definition of a homogeneous space under \(G\). We define the notion of equivariance as follows:
%
%We view a layer in a neural network as an operator. To ensure network equivariance, we require the operator to be equivariant with respect to the group actions on the corresponding function spaces.
%
\begin{definition}\label{def:equivar} 
Let \(G\) be a connected Lie group with homogeneous spaces \(\gM\) and \(\gN\), and \(\phi\) an operator mapping functions from \(\gM\) to \(\gN\). \(\phi\) is equivariant with respect to \(G\) if, for all functions \(f\), we have: $\forall~h \in G$, $$(\phi \circ \mathcal{L}_h) f = (\mathcal{L}_h \circ \phi) f;$$ where \(\circ\) denotes the composition of operators.
\end{definition}

%%%%%%%%%%%%%%%%%%%%%%%%%%%%%%%%%%%%%%%%%%
\section{Proposed DEGMC diffusion model}
%%%%%%%%%%%%%%%%%%%%%%%%%%%%%%%%%%%%%%%%%%%
%In this section, we present our proposed model, \textit{Equivariant Denoising Diffusion Models based on Riemannian Morphological PDEs} (DEGMC), by outlining all the theoretical and conceptual contributions that distinguish it from classical DDPM.  
%Like the latter, the proposed model is based on the same principle described in Section \ref{sec:pdm}: it consists of two stages, a noise diffusion phase and a denoising (or reverse) process.\\
DEGMC preserves both the DDPM forward process and the ELBO. The reverse process involves using morphological group convolutions and convection operators to obtain an equivariant network for noise prediction.

%-------------------------------------------------------------
\subsection{Equivariant morphological layers model prediction }%
%-------------------------------------------------------------

PDE-G-CNNs were formally introduced in homogeneous spaces with $G$-invariant tensor metric fields on quotient spaces \cite{Diop2024b}. Building on this foundational approach, the proposed DEGMC model combines traditional CNNs with group morphological convolution layers based on Hamilton-Jacobi PDEs defined on Riemannian manifolds and convection layers. The network's output is obtained as a linear combination of these terms. This system represents our stepwise model, which is solved using operator splitting; each step corresponds to one of the preceding terms. Below, we detail each counterpart of the system:  

\subsubsection{Convection term}
%%%%%%%%%%%%%%%%%%%%%%%%%%%%%%%%%
In our architecture, we use the convection part as a learned resampling operator that moves and aligns features in space according to the underlying geometry of the data. The convection term is obtained by solving the following PDE:
\begin{align}
&\dfrac{\partial u}{\partial t} + \alpha u = 0 \text{ in } (\gM,\mathbf{g})\times (0,~\infty)\nonumber\\
& u(\cdot, 0) = f \text{ on } (\gM,\mathbf{g}), \label{eq:convec}
\end{align}
where \(\alpha\) is a $G$-invariant vector field on $(\gM,\mathbf{g})$. The convection (\ref{eq:convec}) is left-invariant under the action of $G$. PDE \eqref{eq:convec} is solved using the method of characteristics, and its solution is given by the result stated below:
\begin{proposition}\label{pr:convec}
The solution of (\ref{eq:convec}) is obtained using the method of characteristics and is given by:
\begin{align*}
u(x, t) & = (\mathcal{L}_{h_{x}^{-1}} f )~ (\gamma_{c}(t)^{-1} x_{0}) \notag \\ & = f(h_{x} \gamma_{c}(t)^{-1} x_{0}) = f(h_{x}  \gamma_{-c}(t) x_{0}), \label{eq:c3eq16}
\end{align*}
where $h_{x} \in G$ satisfies $h_{x} x_{0} = x$ for a fixed $x_0 \in M$, and $\gamma_{c}: R \to G$ is the exponential curve such that $\gamma_{c}(0) = e$ and
\begin{equation*}
\frac{\partial}{\partial t} (\gamma_{c}(t)x)(t) = c (\gamma_{c}(t)x).\label{eq:c3eq17}
\end{equation*}
\ie~$\gamma_{c}$ is the exponential curve in the group $G$ that induces the integral curves of the $G$-invariant vector field $c$ on $\mathcal{M}$ when acting on the elements of the homogeneous space.
\end{proposition}
\begin{proof}
See \cite{Smets13July2022}.
\end{proof}

%%%%%%%%%%%%%%%%

\subsubsection{Equivariant group morphological convolution}
%%%%%%%%%%%%%%%%%%%%%%%%%%%%%%%%%%%%%%%%%%%%%%%%%%%%%%%%%%%%%%%%
%The connection between multiscale morphological dilations and erosions had already been established by solving a first-order Hamilton–Jacobi type PDE in \(\R^n\). Their extensions in Riemannian manifolds can be provided by properly defining the related Hamiltonian. 
Here, we introduce the notion of group morphological convolutions, which are equivariant in the sense of Definition~\ref{def:equivar}. These are derived from the multiscale morphological dilations and erosions, which are the viscosity solutions of first-order Hamilton–Jacobi type PDEs defined on compact Riemannian manifolds.

Let \((\gM, \mathbf{g})\) be a compact, connected Riemannian manifold equipped with a metric \(\mathbf{g}\), and let \(f: (\gM, \mathbf{g}) \longrightarrow \R\). Let \(T\gM\) denote the tangent bundle of \((\gM, \mathbf{g})\), and let \(L : T\gM \to \R\) be a Lagrangian function. Let \(T^\ast\gM\) denote the cotangent bundle of \((\gM, \mathbf{g})\), and let us define the Hamiltonian \(H : T^\ast\gM \to \R\) associated with the Lagrangian \(L\) by:
\[H(x, q) = \underset{v \in T_x \gM}{\sup} \{q(v) - L(x, v)\}.\] 
We define the related Hamilton–Jacobi PDE in a Riemannian manifold as follows:
\begin{align}
\label{pde}
&\dfrac{\partial u}{\partial t} + H(x, \nabla u) = 0 \quad \text{in } (\gM, \mathbf{g}) \times (0, +\infty);\notag \\ 
&u(\cdot, 0)  = f \text{ on } (\gM, \mathbf{g}).
\end{align}
PDE~(\ref{pde}) admits unique viscosity solutions \cite{Fathi2008,Diop2021}. Multiscale morphological erosions (\resp~multiscale morphological dilations) are obtained by taking \(H = \lVert\nabla_{\mathbf{g}} u\rVert_{\mathbf{g}}^k\) (\resp~\(H = -\lVert\nabla_{\mathbf{g}} u\rVert_{\mathbf{g}}^k\)) in \eqref{pde}:

\begin{proposition}
\label{prop:conv}
Let \(k>1\), $c_k = \frac{k-1}{k^{\frac{k}{k-1}}}$, and a continuous function \(f \in C^0( (\gM,\mathbf{g}), \R)\). The unique viscosity solution of the Cauchy problem:
\begin{align}
\label{eq:c3eq28}
&\dfrac{\partial u}{\partial t} + \|\nabla_{\mathbf{g}} u\|_{\mathbf{g}}^k  = 0 \quad \text{in } (\gM,\mathbf{g}) \times (0, +\infty); \notag \\
& u(\cdot,0) = f \text{ on } (\gM,\mathbf{g}),
\end{align} 
is given by:
\begin{equation}
\label{multi_ero}
u(t,x)= \inf_{h \in G} \left\{ f\big(\varphi_h(x_0)\big) + c_k  \frac{d_\mathbf{g}\big(\varphi_{h^{-1}}(x),x_0\big)^{\frac{k}{k-1}}}{t^{\frac{1}{k-1}}} \right\}.
\end{equation} 
\end{proposition}
\begin{proof}
See \cite{Diop2024b}.
\end{proof}
Morphological multiscale Riemannian dilations at scale $t$ are obtained by reversing time:
\begin{align}\label{eq:dil}
& \dfrac{\partial w}{\partial t} - \|\nabla_{\mathbf{g}} w\|_{\mathbf{g}}^k = 0 \quad \text{in } (\gM,\mathbf{g}) \times (0,+\infty); \notag \\ 
& w(\cdot,0)  = f \text{ on } (\gM,\mathbf{g}).
\end{align} 
The viscosity solution of the Cauchy problem \eqref{eq:dil} is obtained in a similar way and is given by:
\begin{equation}
\label{multi_dil}
w(t,x)= \sup_{h \in G} \left\{ f\big(\varphi_h(x_0)\big) - c_k  \frac{d_\mathbf{g}\big(\varphi_{h^{-1}}(x),x_0\big)^{\frac{k}{k-1}}}{t^{\frac{1}{k-1}}} \right\}.
\end{equation}
We introduce the notion of group Riemannian morphological convolution below.
\begin{definition}
\label{def:conv_gp}
Let \(f,b : (\gM, \mathbf{g}) \longrightarrow \R\). The group morphological convolutions $\triangledown$ and $\vartriangle$ between \(b\) and \(f\) are respectively defined for all \(x \in (\gM, \mathbf{g})\) by:
\begin{align*}
&(f \triangledown b)(x) = \inf_{h \in G} \{f(\varphi_h(x_0)) + b(\varphi_{h^{-1}}(x))\} \text{ (erosion)}\\
&(f \vartriangle b)(x) = \sup_{h \in G} \{f(\varphi_h(x_0)) - b(\varphi_{h^{-1}}(x))\} \text{ (dilation)}
\end{align*}
\end{definition}
Operators $\triangledown$ and $\vartriangle$ are dual due to the duality between the infimum and supremum operations, we have:
$$f \triangledown b = -(f \triangledown (-b)).$$ The morphological operations are equivariant with respect to $G$ in the sense of Definition~\ref{def:equivar}, due to the following result:
\begin{proposition}\label{inv_geod}
Let \(x, y \in (\gM,\mathbf{g})\) such that \(\varphi_h(y)\) lies outside the cut locus of \(\varphi_h(x)\). Then, for all \(h \in G\), we have: $$d_{\mathbf{g}}(x,y) = d_{\mathbf{g}}\big(\varphi_h(x), \varphi_h(y)\big).$$
\end{proposition}
\begin{proof}
See Appendix~\ref{apA}.
\end{proof}
For multiscale operations, we consider the family of functions $(b^{k}_{t})$ defined by: $$b^{k}_{t}(\cdotp) = c_k \dfrac{d_g(x_0,\cdot)^{\frac{k}{k-1}}}{t^{\frac{1}{k-1}}}.$$ 
The case $k=2$ corresponds to quadratic structuring functions. Letting \(k>1\) allows us to deal with more general structuring functions than quadratic ones, leading to a better handling of thin data structures.

Morphological multiscale Riemannian erosions (\ref{multi_ero}) and dilations (\ref{multi_dil}) at scale $t$ can now be formulated using the group convolutions as follows:
\begin{equation*}
u(t,x)= (f \triangledown  b_{t}^{k})(x) \text{ and } w(t,x)= (f \vartriangle b_{t}^{k})(x).
\end{equation*}

\subsubsection{Hyperbolic Ball example: metric, invariance and embedding}
%%%%%%%%%%%%%%%%%%%%%%%%%%%%%%%%%%%%%%%%%%%%%%%%%%%%%%%%%%%%%%%%%%%%%%%%%%%

For computational purposes, let us consider the hyperbolic ball $\B^{n}$, as it is a compact Riemannian manifold and provides a natural framework to represent the equivariance with respect to the Euclidean group $E(n)$ of translations, rotations, reflections, and permutations in $\R^n$.  $\B^{n}$ has a negative curvature that allows for effective capture of hierarchical and non-local relationships between points, while ensuring that distances, invariant under the transformations of the group $E(n)$ are preserved. This facilitates the definition of stable equivariant operators and improves learning in neural networks based on PDE-GCNNs. $\B^{n}$ is defined by:
\begin{equation*}
\B^{n} = \left\{ \left(x_{1}, \ldots,x_n\right) \in \R^{n} \mid \sum_{i=0}^{n} x_{i}^{2} < 1 \right\}.
\end{equation*}
Let us take $\gM=\B^{n}$ endowed with the metric $g$ defined as follows:
\begin{equation*}
\mathbf{g} = \frac{4(dx_1^2 + \ldots + dx_n^2)}{(1 - \lVert x \rVert^2)^2},
\end{equation*}
where $\lVert \cdot \rVert$ represents the Euclidean norm in $\R^n$. Next, we show that the hyperbolic distance $\mathrm{d}_{\B^{n}}$ induced by $g$ is invariant under translations, rotations, reflections, and permutations. We also show an embedding of $\R^n$ into $\B^n$, which will preserve data structures within the hyperbolic ball $\B^n$, enabling the desired equivariance.

The length of a curve $\gamma : [a, b] \rightarrow \B^{n}$ is given by:
\begin{align*}
L(\gamma) & = \int_{a}^{b} \sqrt{g(\gamma'(t), \gamma'(t))} \ud t \notag \\ &  = \int_{a}^{b} 2 \frac{\sqrt{\gamma'_1(t)^2 + \ldots + \gamma'_n(t)^2}}{\sqrt{1 - (\gamma_1(t)^2 + \ldots + \gamma_n(t)^2)}} \ud t,
\end{align*}
where $\gamma(t) = (\gamma_1(t), \ldots, \gamma_n(t))$.  

The distance between two points $x, y \in \B^{n}$ is the infimum over all curves that join $x$ and $y$. Then, the hyperbolic distance $d_{\B^{n}}(x, y)$ between $x$ and $y$ is given by:
\begin{equation*}
\cosh\mathrm{d}_{\B^{n}}(x, y) = 1 + \frac{2\lVert x - y \rVert^2}{(1 - \lVert x \rVert^2)(1 - \lVert y \rVert^2)},
\end{equation*}
and thus, we derive:
\begin{equation*}
\mathrm{d}_{\B^{n}}(x, y) = \operatorname{Argcosh} \left(1 + \frac{2\lVert x - y \rVert^2}{(1 - \lVert x \rVert^2)(1 - \lVert y \rVert^2)}\right).
\end{equation*}
Since $\mathrm{d}_{\B^n}$ depends only on the Euclidean norm, which is invariant under Euclidean isometries, $\mathrm{d}_{\B^n}$ is invariant under all elements of $E(n)$, as stated in the following result:  
\begin{proposition}
\label{inv_dbn}
 $\mathrm{d}_{\B^{n}}$ is invariant under Euclidean transformations.
 \end{proposition}
 \begin{proof}
 See Appendix~\ref{apB}.
 \end{proof}
The next result is the embedding of $\R^n$ into $\B^n$. To prove it, let us consider the mapping $\Phi$ defined from $\R^n$ to $\B^{n}$ by:
\begin{equation*}
\begin{array}{rl}
    \Phi: & \R^n \longrightarrow \B^{n};~x \mapsto \dfrac{x}{\sqrt{1+\Vert x\Vert^{2}}},
\end{array}
\end{equation*}
where $\Vert \cdot \Vert$ denotes the Euclidean norm in $\R^n$. $\Phi$ is well-defined because, $\forall x \in \R^n$, we have:
\begin{equation*}
\Vert \Phi(x) \Vert^2 = \dfrac{\Vert x \Vert^2}{1+\Vert x \Vert^2} < 1.
\end{equation*}

\begin{proposition}
\label{map}
$\Phi$ is an embedding of $\R^n$ into $\B^{n}$.
\end{proposition}
 \begin{proof}
 See Appendix~\ref{apC}.
 \end{proof}
The inverse mapping $S$ (or inverse stereographic projection) is defined by:
\begin{equation*}   
\begin{array}{rl}
    S: & \B^n \to \R^{n} \\
       & x \mapsto \dfrac{x}{\sqrt{1-\Vert x\Vert^{2}}}
\end{array}
\end{equation*}

which is well-defined as long as $\Vert x \Vert < 1$, a condition always satisfied within $\B^n$.

%-----------------------------------------------------------
\subsection{GMCUnet Network Architecture} \label{sec:archigmcunet}%------------------------
%-----------------------------------------------------------

Our architecture (Fig.~\ref{fig:archi}(a)) is based on a classical U-Net structure, commonly used for noise prediction in DDPM \cite{ho2020denoising}. We nonetheless incorporate the previously defined PDE layer into this architecture. In our numerical experiments, we consider $\gM = \B^n$, which enables the construction of a new network that is equivariant with respect to the group $E(n)$.

A diffusion U-Net typically consists of three main components: an encoder (\textit{Downsampling}), a decoder (\textit{Upsampling}), and a middle block (\textit{Middle Block}). The latter, located between the encoder and the decoder, corresponds to the lowest spatial resolution and the highest number of channels. It plays a crucial role in merging, transforming, and refining the representations extracted by the encoder before their reconstruction by the decoder. It often consists of residual blocks for deeper feature extraction and may include an attention module to capture long-range dependencies. In DEGMC, we propose a novel modification in the middle block. Specifically, we replace the classical ResNetBlocks, commonly used in standard U-Nets, with \textit{ResnetCDEBlocks}. As illustrated in Fig.~\ref{fig:archi}(b), these blocks enhance predictive capacity while introducing explicit equivariance with respect to the group $E(n)$.

To the best of our knowledge, the combination of a diffusion U-Net and residual blocks of the CDE type (Convection, Dilation, and Erosion) has not yet been reported in the literature. Therefore, the proposed GMCUnet architecture constitutes an original contribution aimed at enhancing both the robustness and expressiveness of diffusion models. Its use in the reverse diffusion process is described in Algorithm~\ref{alg:reverse}.
%%----------------------------algorithm---------------------------%
\begin{algorithm}[tb]
\caption{Reverse Diffusion Process of DEGMC}
\label{alg:reverse}
\begin{algorithmic}

\STATE {\bfseries Input:} initial noise $n_T \sim \mathcal{N}(0, I)$
\STATE {\bfseries Output:} generated sample $n_0$

\FOR{$t = T$ {\bfseries down to} $1$}

    \STATE Compute time embedding $e_t \leftarrow \mathrm{MLP}(t)$
    
    \STATE Apply the GMCUnet network:
    \STATE \quad Encoder :
    \STATE \quad\quad $h_t \leftarrow \mathrm{Downsampling}(n_t, e_t)$
    \STATE \quad Middle Block:
    \STATE \quad\quad $h_t \leftarrow \mathrm{ResnetCDEBlocks}(h_t)$
    \STATE \quad\quad $h_t \leftarrow \mathrm{Attention}(h_t)$
    \STATE \quad\quad $h_t \leftarrow \mathrm{ResnetCDEBlocks}(h_t)$ \hfill {\footnotesize (×2)}
    \STATE \quad Decoder:
    \STATE \quad\quad $\hat{\epsilon}_t \leftarrow \mathrm{Upsampling}(h_t, e_t)$
    
    \STATE Compute the mean of the reverse process:
    \STATE \quad $\mu_t =
        \frac{1}{\sqrt{\alpha_t}}
        \left(
        n_t -
        \frac{1 - \alpha_t}{\sqrt{1 - \bar{\alpha}_t}}
        \hat{\epsilon}_t
        \right)$
    
    \STATE Sample the previous state:
    \STATE \quad $n_{t-1} \sim \mathcal{N}(\mu_t, \beta_t I)$

\ENDFOR

\STATE {\bfseries Return} $n_0$

\end{algorithmic}
\end{algorithm}
%----------------------------algorithm---------------------------%

%\begin{figure}[h]
    %\centering
    %\begin{subfigure}{0.45\textwidth}
        %\centering
        %\includegraphics[width=4cm]{arch1.png}
        %\caption{Architecture of the PDEUnet network}
   % \end{subfigure}%
    %\hfill
    %\begin{subfigure}{0.45\textwidth}
        %\centering
        %\includegraphics[width=6cm]{arch2.png}
        %\caption{Detail of the ResnetCDEBlock}
        %\label{fig:arch2}
    %\end{subfigure}
    %\caption{Architecture of the equivariant PDE-G-CNN network for prediction, called PDEUnet, along with the detail of the ResnetCDE block using the three PDEs: convection, dilation, and erosion.}
    %\label{fig:archipdeunet}
%\end{figure}
%%%%%%%%%%%%%%%%%%%%%%%fin de proposition%%%%%%%%%%%%%%%%%%%%% 

%\section{Related Work}\label{sec:relatework}

%\subsection{Diffusion model}
%\subsection{Equivariance}
%%%%%%%%%%%%%%%%%%%%%%%%%%%%%
\section{Experiments}%%%%%%%%
%%%%%%%%%%%%%%%%%%%%%%%%%%%%%

To evaluate the performance of our diffusion model, numerical experiments were conducted on the \textbf{MNIST}, \textbf{Rotated MNIST}, and \textbf{CIFAR-10} datasets. For MNIST and Rotated MNIST, the model is trained using a batch size of 64 images for a total of $60,000$ training iterations. For CIFAR-10, the training is performed with a batch size of 128 images over $80,000$ iterations. In all cases, optimization relies on the Adam algorithm ($\beta_1 = 0.9, \beta_2 = 0.99$) with a fixed learning rate of $1 \times 10^{-4}$.
To stabilize training and improve the quality of the generated images, an exponential moving average (EMA) of the model parameters is applied every $10$ iterations with a decay factor of $0.995$.

Test image generation is periodically performed to monitor the evolution of the quality of the generated samples. For MNIST and Rotated MNIST, the sampling is carried out every $100 $ iterations, with $25$ images generated at each step. For CIFAR-10, images are generated every $1,000$ iterations, with $64$ images produced at each step.

A quantitative evaluation is performed using the Fréchet Inception Distance (FID) metric. FID scores are computed using features extracted from an Inception network of dimension $2048$ on $2,500$ generated images for MNIST and Rotated MNIST and on $5,000$ generated images for CIFAR-10. Additional metrics are considered to provide a more detailed evaluation of model performance.

In this study, we choose to compare our approach exclusively with the baseline DDPM model. There are two main reasons for this choice. First, the DDPM model is foundational for diffusion-based approaches and is a standard reference in much of the existing literature. Therefore, it provides a relevant baseline for objectively assessing the improvements introduced by our method. Second, rather than attempting to outperform all existing diffusion model variants, our objective is to quantify the impact of introducing equivariance through group morphological convolutions on the original model's performance. By limiting the comparison to this reference model, we can isolate and rigorously analyze the true contribution of our approach.

The following presents the results in the form of tables and figures, along with image samples generated for each dataset and model.
%----------------figure des sample mnist-------------------%
\begin{figure}[!ht]
    %\vskip 0.2in
    \centering
    \begin{subfigure}{0.48\columnwidth}
        \centering
        \includegraphics[width=\linewidth]{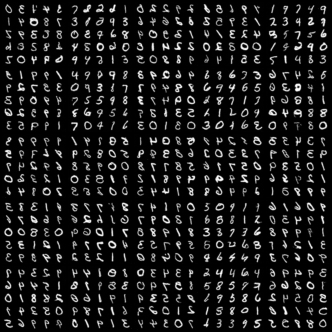}
        \caption{DEGMC (FID = 30.94)}
        \label{fig:ed2rm}
    \end{subfigure}
    \hfill
    \begin{subfigure}{0.48\columnwidth}
        \centering
        \includegraphics[width=\linewidth]{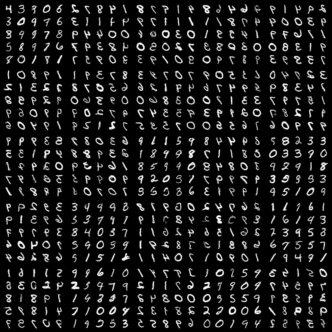}
        \caption{DDPM (FID = 36.41)}
        \label{fig:ddpm}
    \end{subfigure}
    \caption{Best samples generated on MNIST based on the lowest FID scores: DEGMC vs. DDPM.}
    \label{fig:samplemnist}
\end{figure}
%------------------------------------------------------------%
Fig.~\ref{fig:samplemnist} presents the samples obtained at the training iterations corresponding to the best FID scores using DEGMC (Fig.~\ref{fig:ed2rm}) and DDPM (Fig.~\ref{fig:ddpm}). A visual inspection reveals that the overall quality of the generated samples is comparable with respective FID scores of $30.94$ and $36.41$. These results suggest that the DEGMC model has slightly superior generative performance on the MNIST dataset.\\
Table~\ref{tab:metric} reports obtained quantitative results using FID and mean Inception Score (IS) metrics for the generated samples. The results for DEGMC and DDPM show comparable performance in terms of sample quality and diversity on MNIST. Specifically, DEGMC achieves an average FID of $45.14$, which is slightly lower than DDPM's FID average of $46.91$. Furthermore, Fig.~\ref{fig:fidmnist} shows the evolution of the FID during training. It highlights that DEGMC generates higher-quality samples during the first thirty training iterations. DDPM subsequently reaches a comparable level of quality after the thirtieth iteration. These results suggest that DEGMC converges more rapidly toward high-quality sample generation due to its equivariance property, whereas DDPM requires a larger number of iterations to reach similar performance.

%--------- ------tableau des metrics ----------------------%
\begin{table}[t]
\caption{Mean FID and IS scores on MNIST, RotoMNIST, and CIFAR-10.}
\label{tab:metric}
\centering
\small
\setlength{\tabcolsep}{6pt}
\renewcommand{\arraystretch}{1.1}

\begin{tabular}{lcccc}
\toprule
Base &
\multicolumn{2}{c}{DEGMC} &
\multicolumn{2}{c}{DDPM} \\
\cmidrule(lr){2-3} \cmidrule(lr){4-5}
& FID & IS & FID & IS \\
\midrule
MNIST 
& 45.14 & $1.21 \pm 0.24$ 
& 46.91 & $1.21 \pm 0.24$ \\

RotoMNIST 
& 49.30 & $1.33 \pm 0.18$ 
& 54.16 & $1.30 \pm 0.14$ \\

CIFAR-10 
& 64.83 & $1.50 \pm 0.24$ 
& 65.51 & $1.39 \pm 0.12$ \\
\bottomrule
\end{tabular}
\vskip -0.1in
\end{table}

%------------------------------------------------------------%
%----------------figure des sample roto mnist-------------------%
\begin{figure}[!ht]
    %\vskip 0.2in
    \centering
    \begin{subfigure}{0.48\columnwidth}
        \centering
        \includegraphics[width=\linewidth]{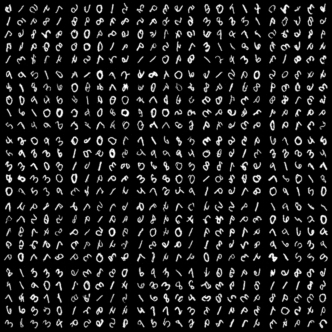}
        \caption{DEGMC (FID = 35.75)}
        \label{fig:ed2rmroto}
    \end{subfigure}
    \hfill
    \begin{subfigure}{0.48\columnwidth}
        \centering
        \includegraphics[width=\linewidth]{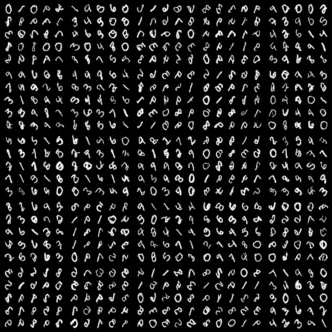}
        \caption{DDPM (FID = 44.74)}
        \label{fig:ddpmroto}
    \end{subfigure}
    \caption{Best samples generated on RotoMNIST based on the lowest FID scores: DEGMC vs. DDPM.}
    \label{fig:sampleroto}
\end{figure}
%------------------------------------------------------------%
On the RotoMNIST dataset, the samples generated by DEGMC (Fig.~\ref{fig:ed2rmroto}) exhibit superior quality, achieving an FID score of $35.75$. In contrast, DDPM (Fig.~\ref{fig:ddpmroto}) obtains a higher FID score of $44.74$. Table~\ref{tab:metric} shows that DEGMC achieves better FID and IS scores, demonstrating its ability to generate higher-quality and more diverse samples than DDPM.

%----------------figure des sample cifar-------------------%
\begin{figure}[!ht]
    %\vskip 0.2in
    \centering
    \begin{subfigure}{0.8\columnwidth}
        \centering
        \includegraphics[width=0.8\linewidth]{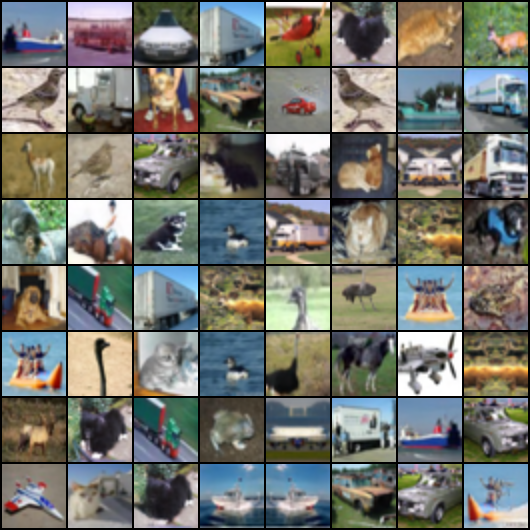} % \linewidth
        \caption{DEGMC (FID = 25.14)}
        \label{fig:ed2rcifar}
    \end{subfigure}\\
    %\hfill
    \begin{subfigure}{0.8\columnwidth}
        \centering
        \includegraphics[width=0.8\linewidth]{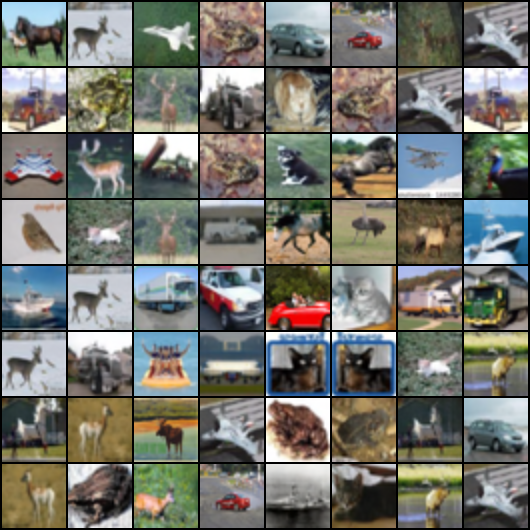}
        \caption{DDPM (FID = 25.59)}
        \label{fig:ddpmcifar}
    \end{subfigure}
    \caption{Best samples generated on CIFAR-10 by DEGMC and DDPM during training, selected based on the lowest FID score.}
    \label{fig:samplecifar}
\end{figure}
%------------------------------------------------------------%
Furthermore, Fig.~\ref{fig:fidroto} illustrates the evolution of the FID over training iterations on the RotoMNIST dataset and highlights a significant performance gap between the two models for most of the training process. We observe that DEGMC effectively adapts to the RotoMNIST data, maintaining a performance level comparable to that obtained on MNIST. This clearly demonstrates the impact of equivariance. This property enables the model to perform similarly on untransformed data and data subjected to group transformations. In contrast, DDPM fails to maintain equivalent performance when trained for the same number of iterations.

Our objective on the CIFAR-10 dataset is to evaluate the performance of our model on real-world color images. Fig.~\ref{fig:samplecifar} presents the best samples generated during training for the two models being compared. Our DEGMC model achieves a slightly lower FID value ($25.14$) than the DDPM baseline model ($25.59$), indicating an improvement in image fidelity and overall visual quality. Table~\ref{tab:metric} summarizes the mean FID and Inception Score (IS) values and reports an average FID of 64.83 for DEGMC, compared to 65.51 for DDPM. These results confirm an overall improvement in generative performance relative to the DDPM model's samples. Furthermore, the higher IS value obtained by our model indicates greater diversity in the generated images and better coverage of the real data distribution.

In conclusion, both models perform similarly on the MNIST dataset, although DDPM requires more training iterations to reach similar performance levels. However, on the Rotated MNIST dataset, DDPM fails to maintain the performance observed on MNIST. In contrast, DEGMC preserves nearly identical generation quality across both datasets. We also observe high-quality image generation on the CIFAR-10 dataset, demonstrating that our model can effectively generalize to real-world color images. This stability in performance is due to DEGMC's intrinsic equivariance with respect to group transformations. This highlights the importance of equivariance as a key property for achieving robust and consistent generation across diverse data types.

%----------------figure des courbes FID-------------------%
\begin{figure}[!ht]
    %\vskip 0.2in
    \centering
    \begin{subfigure}{1\columnwidth}
        \centering
        \includegraphics[width=\linewidth]{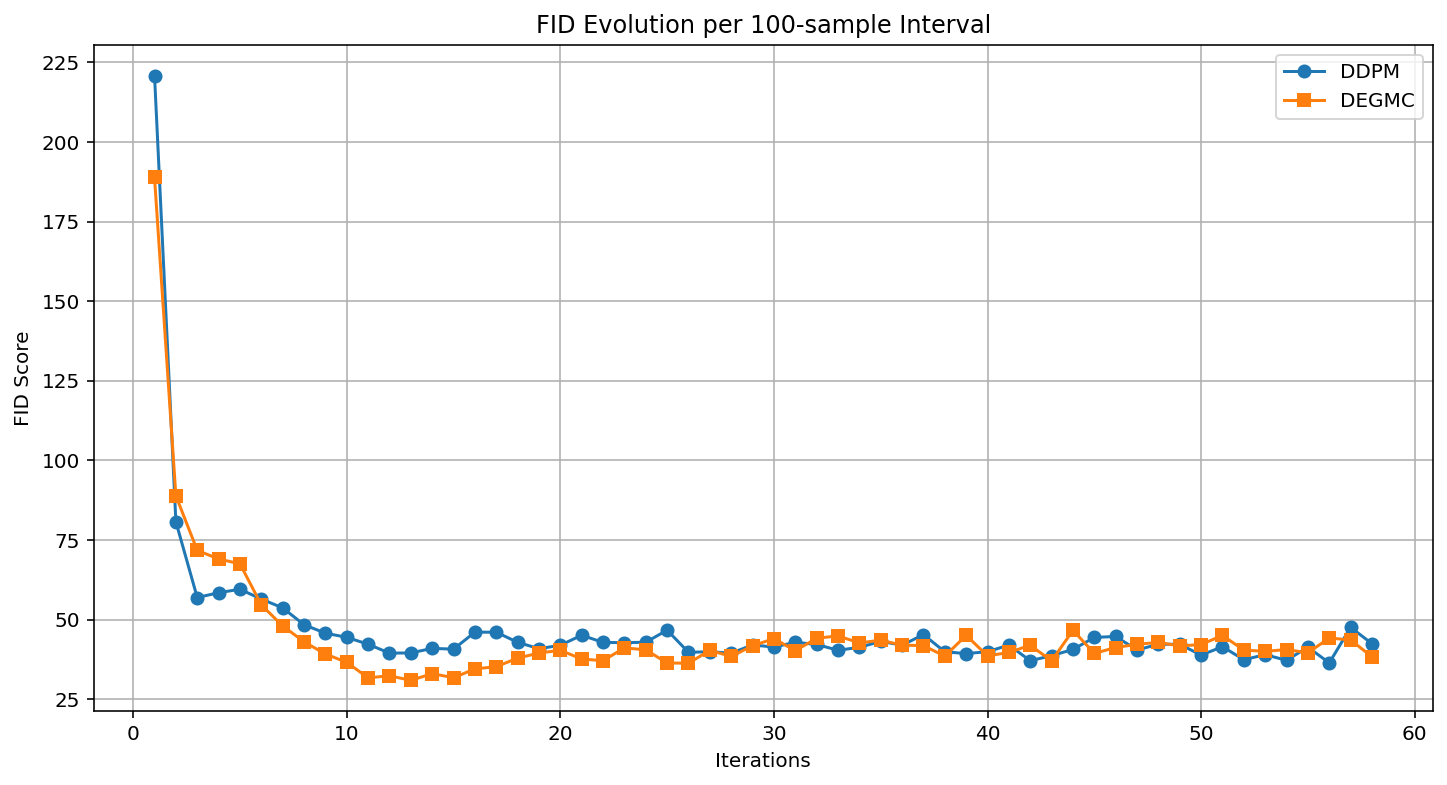}
        \caption{FID on MNIST}
        \label{fig:fidmnist}
    \end{subfigure}\\
    %\hfill
    \begin{subfigure}{1\columnwidth}
        \centering
        \includegraphics[width=\linewidth]{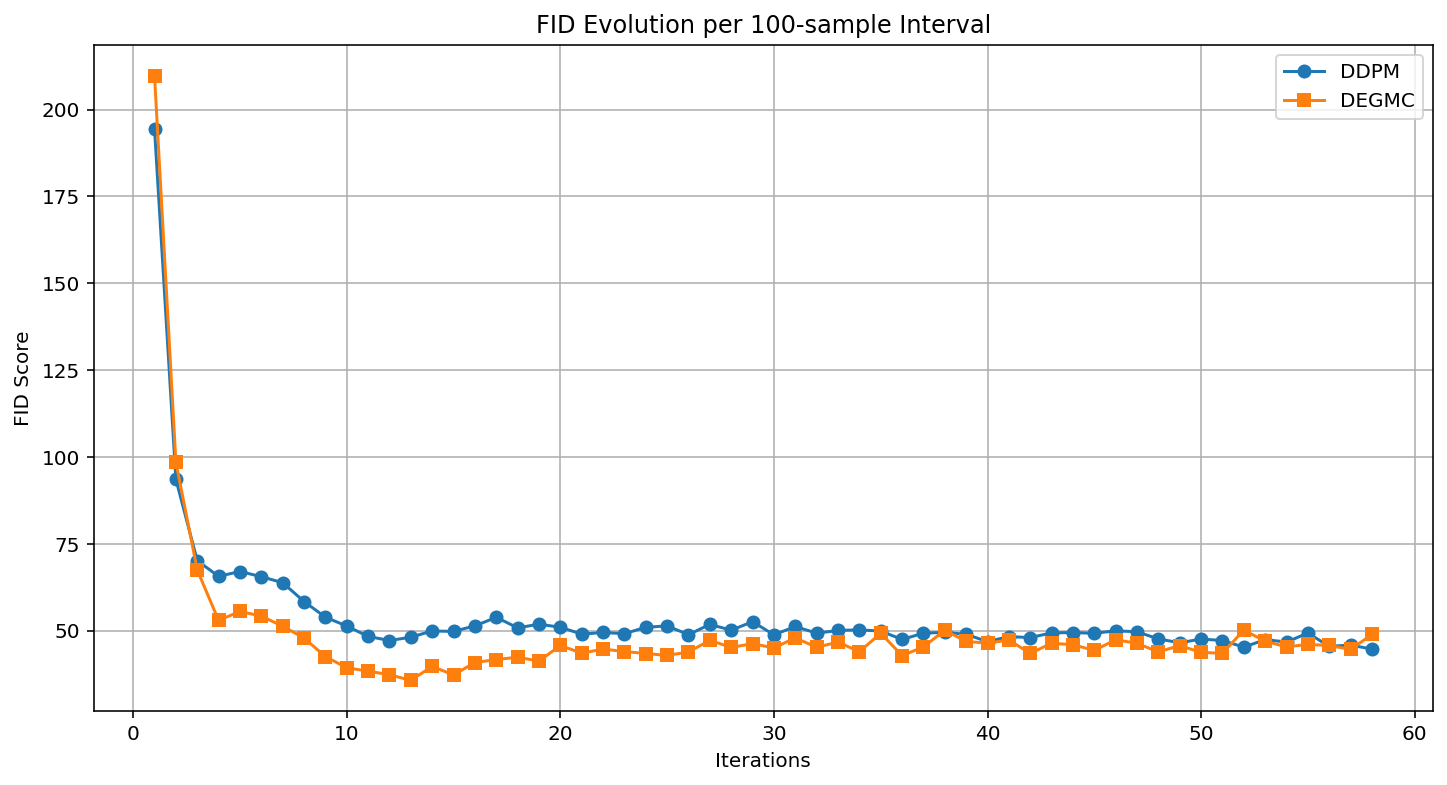}
        \caption{FID on RotoMNIST}
        \label{fig:fidroto}
    \end{subfigure}
    \caption{FID evolution during training using DEGMC and DDPM on MNIST and RotoMNIST.}
    \label{fig:fidcurves}
\end{figure}

%------------------------------------------------------------%

%%%%%%%%%%%%%%%%%%%%%%%%%%%%%
\section{Conclusion}%%%%%%%%
%%%%%%%%%%%%%%%%%%%%%%%%%%%%%
We have proposed here an equivariant denoising diffusion model that integrates equivariant group morphological convolutions on Riemannian manifolds. Not only does the framework preserve the key data symmetries—translations, rotations, reflections, and permutations, but it also enhances geometric feature extraction in the denoising process. Our experiments on the MNIST, RotoMNIST, and CIFAR-10 datasets confirm that the proposed DEGMC achieves faster convergence, superior FID scores, and improved robustness compared to the standard DDPM, particularly under geometric transformations. These findings highlight the potential of introducing equivariant group morphological convolution architectures in diffusion models to produce more interpretable and resilient generative frameworks. Future work will explore scaling DEGMC to higher-dimensional datasets and extending its applicability to $3$D shape generation and molecular modeling.

\appendix

%------------------------------------------------------------------------------
\section{Details on probabilistic diffusion models}\label{annexe:diffusion}%-----------------------------
%-----------------------------------------------------------------
\begin{figure}[!t]
%\vskip 0.2in
\centering
\includegraphics[width=0.5\columnwidth]{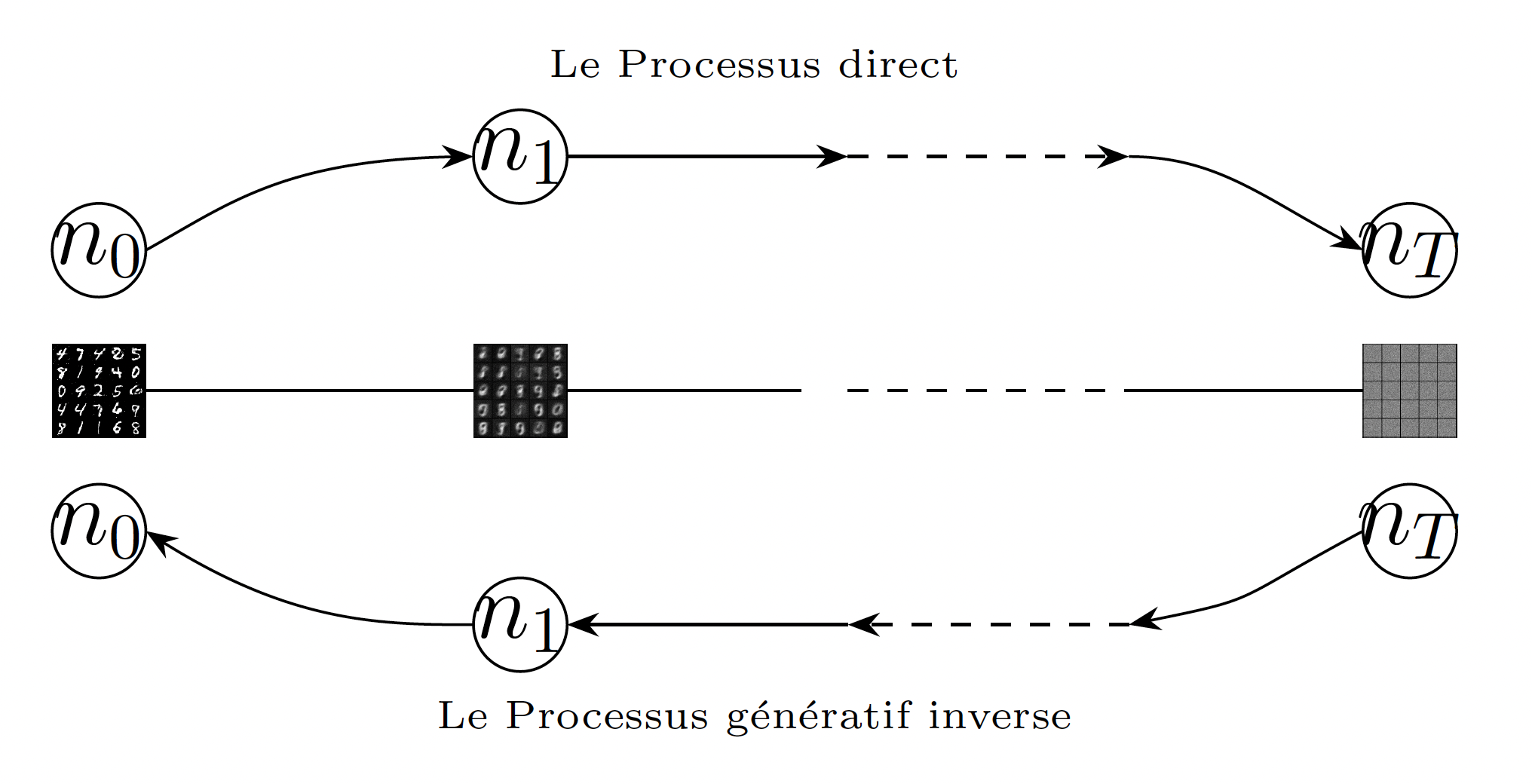}
\caption{Illustration of the forward noising process and the inverse denoising process.}
\label{fig:processus}
\end{figure}

Given observations $x\sim q(x)$, the model operates on latent variables $n_0,\dots,n_T$ of the same dimension as $x$, where $n_0$ corresponds to the observation $x$ and $n_T$ represents standard Gaussian noise (see Fig.~\ref{fig:processus}).

\paragraph{Forward Process.}

The forward process is Markovian; thus, for all $t \in \{0, \dots, T\}$, $n_t$ depends only on $n_{t-1}$ and not on earlier variables \cite{Ghojogh2019}:
\begin{equation}
     q(n_t \mid n_{t-1}, n_{t-2}, \dots, n_0) = q(n_t \mid n_{t-1}) \label{eq:eq11}
\end{equation}
Hence, the joint distribution of this process can be written as:
\begin{equation}
    q(n_1, n_2, \dots, n_T \mid n_0) = \prod_{t = 1}^{T} q(n_t \mid n_{t-1})
\end{equation}
For any $t > s$, the transition distribution from step $s$ to $t$ can be defined using the Gaussian reparameterization of \cite{Kingma2013}, considering a standard Gaussian $\varepsilon \sim \mathcal{N}(0, I)$. Thus, Equation (\ref{eq:eq10}) can be rewritten as:
\begin{equation}
    n_t = \sqrt{\alpha_t} n_{t-1} + \sqrt{1 - \alpha_t} \, \varepsilon \label{eq:reparemetrisation}
\end{equation}
Consequently, for any $t > s$, we have:
\begin{equation}
    q(n_t \mid n_s) = \mathcal{N}(n_t : \sqrt{\alpha_{t/s}} n_s, \Gamma_{t/s} I)
\end{equation}
with $\alpha_{t/s} = \prod_{i = s+1}^{t} \alpha_i$ and $\Gamma_{t/s} = 1 - \alpha_{t/s}$. Relative to the initial data $n_0$:
\begin{equation}
    q(n_t \mid n_0) = \mathcal{N}(n_t : \sqrt{\bar{\alpha}} n_0, \bar{\Gamma} I)
\end{equation}
where $\bar{\alpha} = \prod_{i = 1}^{t} \alpha_i$ and $\bar{\Gamma} = 1 - \bar{\alpha}$.

\paragraph{Inverse Generative Process.} 
\begin{figure}[!t]
%\vskip 0.2in
\centering
\includegraphics[width=\columnwidth]{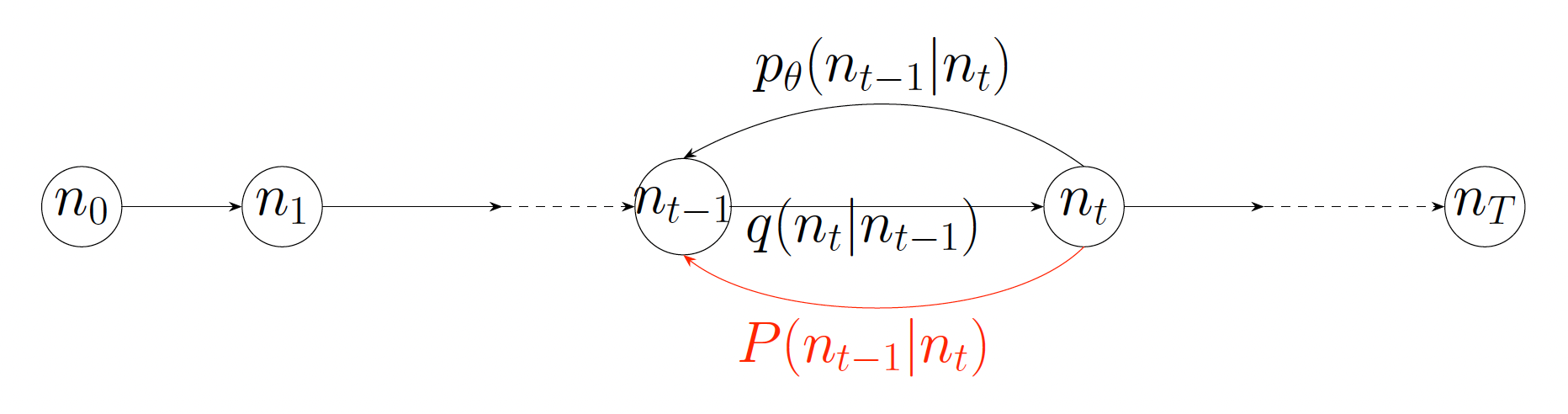}
\caption{Conditional distributions in the forward and generative processes of the diffusion model.}
\label{fig:distribution}
\end{figure}

The generative process is also modeled as a first-order Markov chain, i.e.,
\begin{equation}
    P(n_{t-1} \mid n_t, n_{t+1}, \dots, n_T) = P(n_{t-1} \mid n_t)
\end{equation}
Thus, the joint distribution of the generative process can be written as:
\begin{equation}
    p_{\theta}(n_{0:T}) = P(n_T) \prod_{t=1}^{T} p_{\theta}(n_{t-1} \mid n_t)
\end{equation}
with $P(n_T)$ typically defined as standard Gaussian noise:
\begin{equation}
    P(n_T) = \mathcal{N}(0, I)
\end{equation}

The true distribution is similarly defined as in Equation (\ref{eq:eq10}):
\begin{equation}
    P(n_{t-1} \mid n_t, n_0) = \mathcal{N}(n_{t-1} : \tilde{\mu}(n_t, n_0), \tilde{\sigma}^2 I) \label{eq:dreel}
\end{equation}
Using the same reparameterization technique as in the forward case (see Equation (\ref{eq:reparemetrisation})), we can sample from a standard Gaussian $\varepsilon \sim \mathcal{N}(0, I)$. Then, the mean in Equation (\ref{eq:dreel}) can be expressed as:
\begin{equation}
    \tilde{\mu}(n_t, n_0) = \frac{1}{\sqrt{\alpha_t}} \left(n_t - \frac{\Gamma_t}{\sqrt{1 - \bar{\alpha}_t}} \varepsilon \right)
\end{equation}
Similarly, an expression for $\mu_{\theta}(n_t, t)$ in Equation (\ref{eq:dgeneratif}) is given. Since the learned denoising process is defined from the true denoising process, we have:
\begin{equation}
    \mu_{\theta}(n_t, t) = \frac{1}{\sqrt{\alpha_t}} \left(n_t - \frac{\Gamma_t}{\sqrt{1 - \bar{\alpha}_t}} \varepsilon_{\theta}(n_t, t) \right)
\end{equation}
where $\varepsilon_{\theta}(n_t, t) = \phi(n_t, t)$ is the output of the neural network $\phi$ at iteration $t$.

\paragraph{Variational Lower Bound of the Likelihood.} \label{annexe:elbo}
As mentioned earlier, diffusion models introduce a sequence of latent variables. The data likelihood is written as:
\begin{align*}
 p_{\theta}\left(n_{0}\right)=\int p_{\theta}\left(n_{0}, n_{1: T}\right) \, d n_{1: T}. 
\end{align*}

Direct maximization of this likelihood is intractable; therefore, diffusion models optimize a variational lower bound (ELBO) on the data likelihood:
\begin{align}
\mathcal{L} := \; &\E_{q\left(n_{1: T} \mid n_{0}\right)}\!\left[\log p_{\theta}\left(n_{0} \mid n_{1: T}\right)\right] \notag \\ &- \operatorname{KL}\!\left(q\left(n_{1: T} \mid n_{0}\right) \,\|\, p_{\theta}\left(n_{1: T}\right)\right) 
\;\; \leq \;\; \log p_{\theta}\left(n_{0}\right),
\end{align}
where $\E[\cdot]$ denotes expectation and $\operatorname{KL}(\cdot \| \cdot)$ is the Kullback–Leibler divergence.  

Expanding, we obtain an equivalent expression (see \cite{ho2020denoising,SohlDickstein2015}):
\begin{align*}
\mathcal{L}(\theta) = -\operatorname{KL}\!\left(q\left(n_{T} \mid n_{0}\right) \,\|\, p\left(n_{T}\right)\right) \\
 - \sum_{t=2}^{T} \operatorname{KL}\!\left(q\left(n_{t-1} \mid n_{t}, n_{0}\right) \,\|\, p_{\theta}\left(n_{t-1} \mid n_{t}\right)\right)\\ 
+ \E_{q\left(n_{1:T} \mid n_{0}\right)}\!\left[\log p_{\theta}\left(n_{0} \mid n_{1}\right)\right].
\end{align*}

The ELBO must be maximized with respect to $\theta$. The first KL divergence is independent of $\theta$ and can therefore be ignored during optimization. Hence, maximizing $\mathcal{L}$ reduces to:
\begin{align}
\underset{\theta}{\operatorname{minimize}} \;\;  \sum_{t=2}^{T} \operatorname{KL}\!\left(q\left(n_{t-1} \mid n_{t}, n_{0}\right) \,\|\, p_{\theta}\left(n_{t-1} \mid n_{t}\right)\right) 
 \notag \\ - \E_{q\left(n_{1: T} \mid n_{0}\right)}\!\left[\log p_{\theta}\left(n_{0} \mid n_{1}\right)\right].
\end{align}

This equation trains the inverse distribution $p_{\theta}\!\left(n_{t-1} \mid n_{t}\right)$ to match the true denoising distribution $q\!\left(n_{t-1} \mid n_{t}, n_{0}\right)$ by minimizing their KL divergence. It can thus be used as a loss function for a neural network parameterized by $\theta$, emphasizing the alignment between these two distributions.  

In other words, optimizing the ELBO forces the model to learn a denoising process capable of reversing the progressive diffusion of noise. Training consists of bringing the learned inverse process $p_{\theta}$ closer to the true denoising process $q$, while maximizing the likelihood of the observed data.

%%%%%%%%%%%%%%%%%%%%%%%%%%%%%%%%%%%%%%%%%%%%%%%%%%%%%%%%%%%%
\section{Background on morphological operators and PDEs}%%%%
%%%%%%%%%%%%%%%%%%%%%%%%%%%%%%%%%%%%%%%%%%%%%%%%%%%%%%%%%%%%
Let $b: \R^2\rightarrow\bar\R$ be a concave function, known also as the structuring function or convolution kernel. Let us consider the subset $\E$ of $\Z^2$ and the function $f:\E \rightarrow \bar\R$. 
\begin{definition}
\label{def:morpho}
Morphological dilation and erosion are respectively defined as:
\begin{align}
& f \oplus b(x) = \sup_{y\in\E}[f(y) + b(x-y)] \label{inf_sup:dil} \\
& f\ominus b(x) = \inf_{y\in\E}[f(y) - b(y-x)]. \label{inf_sup:ero}
\end{align}
\end{definition}
Let $B \subseteq \E$ be a bounded set. A flat structuring function (SF) satisfies $b(x)= 0$ if $x\in B$ and $b(x)=-\infty$ if $x\in B^c$. The flat morphological dilation and erosion respectively write:
\begin{equation}
\label{inf_sup:st}
f\oplus B(x) = \sup_{y\in B}[f(x-y)] \mbox{ and } f\ominus B(x) = \inf_{y\in B}[f(x+y)].
\end{equation}
As for an interpretation, erosion shrinks positive peaks, and peaks thinner than the structuring function disappear. One has the dual effects for morphological flat dilation. Both the morphological dilation and erosion are translation invariant.
\begin{definition}
\label{def:increas}
Let $\mathcal{F}$ be a family of real functions defined on $\Omega\subseteq\R^2$. We say that $T:\mathcal{F}\rightarrow \mathcal{F}$ is said to be increasing (monotone) if and only if it satisfies:\\$\forall~f_1, f_2\in\mathcal{F}$ such that $(f_1 \geq f_2$ on $\Omega)$ implies $(T(f_1) \geq T(f_2)$ on $\Omega)$.
\end{definition}
\begin{proposition}
Morphological dilation and erosion satisfy the following duality and adjunction properties:
\begin{enumerate}
 \item duality: $f\oplus b= - (-f\ominus b)$
 \item adjunction: $(f_1\oplus b \leq f_2$ on $E)$ $\Longleftrightarrow (f_1 \leq f_2\ominus b$ on $E)$.
\end{enumerate}
\end{proposition}

Let $(b_t)_{t\geq 0}$ the family of structuring functions defined by using the SF $b$, as follows:% $b_t(x)=t\,b(x/t)$, and for $t=0$ by $0$ for $x=0$ and $-\infty$ otherwise. 
\begin{equation*}
b_{t}(x) = \left\{
\begin{array}{cl}
  t b(x/t) & \text{for } t>0 \\
  0        & \text{for } t=0,\; x=0 \\
  -\infty  & \text{otherwise}.
\end{array}
\right.
\end{equation*}
The family $(b_t)_{t\geq0}$ satisfies the semi-group property:\\ $\forall~s,t\geq 0$, $(b_{s}\oplus b_{t})(x) =$ $b_{s+t}(x,y)$. 
\begin{definition}
Morphological multiscale dilations and erosions are defined as follows:
\begin{align}
& (f\oplus b_t)(x) = \sup_{y\in\E}[f(y) + b_t(x-y)] \label{scale_inf_sup:dil} \\
& (f\ominus b_t)(x) = \inf_{y\in\E}[f(y) - b_t(y-x)]. \label{scale_inf_sup:ero}
\end{align}
\end{definition}
Considering flat structuring function (SF), morphological multiscale dilations and erosions are obtained equivalently  by considering $B_t = tB$ as multiscale SFs. \\

The link between morphological scale-spaces and PDEs was established by running the following PDE that performs multiscale flat dilations and erosions on a given image $f$ \cite{Meyer2000,Schmidt2016}: 
\begin{equation}
\label{cont}
 \partial_t u \pm\|\nabla u\| = 0;~u(\cdotp,0) = f.
\end{equation}
Depending on the shape of SF, different PDEs can be obtained. For instance, considering the sets\\ $S_p= \left\{x=(x_1,x_2)\in \R^2: |x|_p \leq 1\right\}$, where $|\cdotp|_p$ is the $L^p$ norm, one gets:
\begin{itemize}
\item for a square $S_1$: $\partial_t u \pm\| \nabla u\|_1 = 0;~u(\cdotp,0) = f$ 
\item for a dis $S_2$: $\partial_t u \pm\| \nabla u\|_2 = 0;~u(\cdotp,0) = f$
\item for a rhombus $S_\infty$: $\partial_t u \pm\| \nabla u\|_{\infty} = 0;~u(\cdotp,0) = f$.
\end{itemize}
Notice that PDE (\ref{cont}) is a special case of first order Hamilton-Jacobi equation type, which can be formulated in a more general form as follows:
\begin{equation}
\label{pde_hamil}
 \left\{
\begin{array}{l} %c
 \dfrac{\partial u(x,t)}{\partial t} + H
 \left(x,\nabla u(x,t)\right) = 0  \mbox{ on }  \R^n \times (0,+\infty)  \\
u(\cdotp,0) = f \mbox{ on } \R^n.
\end{array}
 \right.
\end{equation}
General Hamilton-Jacobi equation is studied in a viscosity sense, because there is no classical solution for such equations. For a convex Hamiltonian $H$ and some regularity on $f$, the viscosity solution is given by Hopf-Lax formulas \cite{Barron2021,Donato2023}: 
\begin{align}
\label{hlo}
u(x,t) = \inf_{y \in \R^n } \left\lbrace f(y) + t L\left( \dfrac{x - y}{t} \right)\right\rbrace,
\end{align}
where $L$ is the Lagrangian, defined as the Legendre-Fenchel transform of $H$.

%%%%%%%%%%%%%%%%%%%%%%%%%%%%%%%%%%%%%%%%%%%%%%%%%%%%%%%%%%%%
\section{Proof of Proposition~\ref{inv_geod}}\label{apA}%%%%
%%%%%%%%%%%%%%%%%%%%%%%%%%%%%%%%%%%%%%%%%%%%%%%%%%%%%%%%%%%%
For all $x \in  (M,g)$, we refer to \(G\)-invariance of vector fields \(X : x \mapsto T_x M\) if $\forall~h \in G$ and for all differentiable functions \(f\), one has:
\begin{equation}
X(x)f = X(\varphi_h(x))[\mathcal{L}_h f]. \label{eq:c3eq7}
\end{equation}

\begin{definition} A vector field \(X\) on $( M,g)$ is invariant with respect to \(G\) if $\forall~h~\in G$ and $\forall~x \in  (M,g)$, one has:
\begin{equation}
X(\varphi_h(x)) = (\varphi_h)_*X(x).\label{eq:c3eq8}
\end{equation} 
\end{definition}

%It is easy to verify that \eqref{eq:c3eq7} and \eqref{eq:c3eq8} are equivalent and that they imply the following.

\begin{definition} \label{inv_met} A \((0, 2)\)-tensor field \(g\) on \( M\) is \(G\)-invariant if $\forall~h \in G$, $\forall~x \in  M$ and $\forall~v, w \in T_x( M)$, one has:
\begin{equation}
g|_{h}(v,w) = g|_{\varphi_h(x)}((\varphi_h)_* v, (\varphi_x)_* w). \label{eq:c3eq9}
\end{equation}
\end{definition}
It follows from Definition~\ref{inv_met} that properties derived from metric tensor field \(G\) invariance and vector field \(G\) invariance are the same.
\begin{definition}
\label{dist}
Let $(M,g)$ a connected Riemannian manifold, $x, y \in  (M,g)$. The distance between $x$ and $y$ is defined as follows:
\begin{align}
\mathrm{d}_{g}(x, y) = \inf_{{\gamma ~\in ~\Gamma_{t}(x,y)}} \int_{0}^{t}\sqrt{g|_{\gamma(t)}(\dot{\gamma}(t),\dot{\gamma}(t))} \ud t,
\end{align}
with $\Gamma_{t}(x,y) =\{\gamma :[0,t]\longrightarrow  (M,g)~~ \text{of class}~~ C^1 ,\gamma (0)= x ~~\text{and}~~ \gamma(t)=y\}$.
\end{definition}

\begin{definition} The cut locus is defined as the set of points $x\in M$ (or $h\in G$) from which the distance map is not smooth (except at $x$ or $h$).
\end{definition}
\begin{proof}
Let us perform a left multiplication by $h$ in one direction and by $h^{-1}$ in the other direction. A bijection can then be established between $C^1$ curves connecting $x$ to $y$ and connecting $\varphi_h(x)$ to $\varphi_h(y)$. Thus, we have:

\begin{align*}
\mathrm{d}_g\big( \varphi_h(x),\varphi_h(y)\big) &= \inf_{{\beta ~\in ~\Gamma_{t}(\varphi_h(x), \varphi_h(y))}}\int_{0}^{t}\sqrt{g|_{\beta(t)}(\dot{\beta}(t),\dot{\beta}(t))} \ud t,\\
&= \inf_{{h\gamma ~\in ~\Gamma_{t}(\varphi_h(x),\varphi_h(y))}}\int_{0}^{t}\sqrt{g|_{h\gamma(t)}\Big(\varphi_h(\dot{\gamma}(t)), \varphi_h(\dot{\gamma}(t))\Big)} \ud t ~~~~~\text{with} ~~ \gamma \in \Gamma_{t}(\varphi_h(x),\varphi_h(y))\\
&= \inf_{{h\gamma ~\in ~\Gamma_{t}(\varphi_h(x),\varphi_h(y))}}\int_{0}^{t}\sqrt{g|_{h\gamma(t)}\Big((\varphi_h)_{*}\dot{\gamma}(t),(\varphi_h)_{*}\dot{\gamma}(t)\Big)} \ud t\\
&= \inf_{{h\gamma ~\in ~\Gamma_{t}(\varphi_h(x),\varphi_h(y))}}\int_{0}^{t}\sqrt{g|_{\gamma(t)}(\dot{\gamma}(t),\dot{\gamma}(t))} \ud t ~~~~~~~~~~~~~~\text{by \eqref{eq:c3eq9}}\\
&= \inf_{{\gamma ~\in ~\Gamma_{t}(x,y)}}\int_{0}^{t}\sqrt{g|_{\gamma(t)}\big(\dot{\gamma}(t),\dot{\gamma}(t)\big)} \ud t = \mathrm{d}_g(x,y)  \qed
\end{align*}
\end{proof}

%%%%%%%%%%%%%%%%%%%%%%%%%%%%%%%%%%%%%%%%%%%%%%%%%%%%%%%%%%%%
\section{Proof of Proposition~\ref{inv_dbn}}\label{apB}%%%%
%%%%%%%%%%%%%%%%%%%%%%%%%%%%%%%%%%%%%%%%%%%%%%%%%%%%%%%%%%%%
\begin{proof}
The case $n=2$ is trivial. Let us prove the result for $n=3$; the general case follows the same way.\\

\noindent $\bullet$ \textbf{Rotations and Reflections}\\

\noindent Let $\{\vec{u},\vec{v},\vec{w}\}$ be an orthonormal basis of $\R^3$.  
Define $R_{\theta}$ in this basis as:
\begin{equation}
R = \begin{pmatrix}
-1 & 0 & 0 \\
0 & \cos{\theta} & -\sin{\theta} \\
0 & \sin{\theta} & \cos{\theta}
\end{pmatrix},
\end{equation}
which represents an anti-rotation by angle $\theta$ around $\vec{u}$ (a composition of rotation and reflection).  
Applying $R_{\theta}$ to $X = (x,y,z)$ yields:
\begin{equation}
R_{\theta}X =
\begin{pmatrix}
-x \\
y\cos{\theta} - z\sin{\theta} \\
y\sin{\theta} + z\cos{\theta}
\end{pmatrix}.
\end{equation}
Its Euclidean norm satisfies:
\begin{equation}
\Vert R_{\theta}X \Vert^2 = \Vert X \Vert^2,
\end{equation}
and similarly $\Vert R_{\theta}X - R_{\theta}Y \Vert = \Vert X-Y \Vert$.  
Substituting these into the expression for $\mathrm{d}_{\B^3}$, we obtain:
\begin{equation}
\mathrm{d}_{\B^{3}}(R_{\theta} X, R_{\theta} Y) = \mathrm{d}_{\B^{3}}(X,Y).
\end{equation}

\noindent $\bullet$ \textbf{Permutations}\\

\noindent Let us represent the group of permutations of $\{1,2,3\}$ as follows:
\begin{equation}
\sigma =
\begin{pmatrix}
1 & 2 & 3\\
\sigma(1) & \sigma(2) & \sigma(3)
\end{pmatrix},
\end{equation}
with $\sigma(1)=2, \sigma(2)=3, \sigma(3)=1$.  
It follows that $\Vert \sigma X\Vert =\Vert X\Vert$ and $\Vert \sigma X - \sigma Y\Vert =\Vert X-Y\Vert$, hence:
\begin{equation}
\mathrm{d}_{\B^{3}}(\sigma X, \sigma Y) = \mathrm{d}_{\B^{3}}(X,Y). \qed
\end{equation}
 
\end{proof}

%%%%%%%%%%%%%%%%%%%%%%%%%%%%%%%%%%%%%%%%%%%%%%%%%%%%%%%%%%%%
\section{Proof of Proposition~\ref{map}}\label{apC}%%%%
%%%%%%%%%%%%%%%%%%%%%%%%%%%%%%%%%%%%%%%%%%%%%%%%%%%%%%%%%%%%

\begin{proof}
$\Phi$ is well-defined and continuous. Next, we show that $\Phi$ is an injection and a $\mathcal{C}^k$-diffeomorphism.  \\

\noindent$\bullet$ \textbf{Injectivity of $P$}\\

\noindent Let $x, y \in \R^n$. We assume that $\Phi(x)=\Phi(y)$, we need to show that $x=y$.
$$
\begin{aligned}
  \Phi(x) = \Phi(y)
  &\Longrightarrow \dfrac{x}{\sqrt{1+\Vert x\Vert^{2}}} = \dfrac{y}{\sqrt{1+\Vert y\Vert^{2}}}\\
  &\Longrightarrow \dfrac{\Vert x\Vert}{\sqrt{1+\Vert x\Vert^{2}}} = \dfrac{\Vert y\Vert}{\sqrt{1+\Vert y\Vert^{2}}}\\
  &\Longrightarrow \Vert x\Vert^2 = \Vert y\Vert^2 \\
  &\Longrightarrow x = y.
\end{aligned}
$$
Hence, $\Phi$ is injective.\\

\noindent$\bullet$ \textbf{$\mathcal{C}^k$-diffeomorphism property of $\Phi$:}\\

\noindent For $x=(x_i)_{i=1}^n \in \R^n$, we have
\begin{equation}
\Phi(x) = \dfrac{x}{\sqrt{1+\Vert x\Vert^2}} = \left( \dfrac{x_i}{\sqrt{1+\Vert x\Vert^2}} \right)_{i=1}^n,
\end{equation}
and we denote $\Phi_i(x) = \dfrac{x_i}{\sqrt{1+\Vert x\Vert^2}}$. Then:
$$
\begin{aligned}
    \frac{\partial}{\partial x_j}\Phi_i(x) &= \frac{\delta_{ij}}{\sqrt{1+\Vert x\Vert^2}} - \frac{x_i x_j}{(1+\Vert x\Vert^2)^{3/2}},
\end{aligned}
$$
where
\begin{equation}
\delta_{ij} = \begin{cases}
1 & \text{if } i=j,\\
0 & \text{otherwise,}
\end{cases}
\end{equation}
is the Kronecker symbol.  
Thus, the Jacobian matrix of $\Phi$ is:
\[
J_{\Phi(x)}= \dfrac{1}{\sqrt{1+\Vert x\Vert^2}}\left(I - \dfrac{x\otimes x}{1+\Vert x\Vert^2}\right),
\]
where $\otimes$ denotes the tensor product.  
In $\R^3$, we obtain:
\[
\det(J_{\Phi(x)}) = 1 - \dfrac{\Vert x\Vert^2}{\sqrt{1+\Vert x\Vert^2}} \neq 0 \quad \forall x\in \R^3,
\]
which shows that $J_{\Phi(x)}$ is invertible. Hence, $\Phi$ is a diffeomorphism onto its image, and therefore, it is an embedding of $\R^n$ into $\B^n$. 
\end{proof}
%%%%%%%%%%%%%%%%%%%%%%%%%%%%%%%%%%%%%%%%%%%%%%%%%%%%%%%%%%%%%%%%%%
\section{Additional Qualitative Results} \label{addresults} %%%%%%
%%%%%%%%%%%%%%%%%%%%%%%%%%%%%%%%%%%%%%%%%%%%%%%%%%%%%%%%%%%%%%%%%%
%----------------figure des sample mnist-------------------%
\begin{figure}[!ht]
    %\vskip 0.2in
    \centering
    \begin{subfigure}{0.48\columnwidth}
        \centering
        \includegraphics[width=\linewidth]{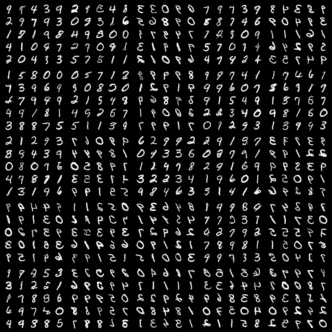}
        \caption{DEGMC (FID = 36.31)}
        \label{fig:ed2rm30}
    \end{subfigure}
    \hfill
    \begin{subfigure}{0.48\columnwidth}
        \centering
        \includegraphics[width=\linewidth]{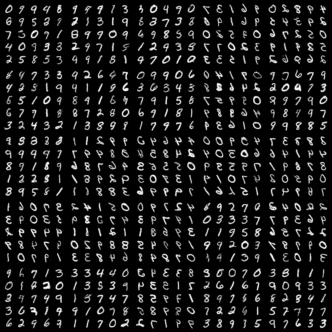}
        \caption{DDPM (FID = 41.39)}
        \label{fig:ddpm30}
    \end{subfigure}
    \caption{Generated image samples on MNIST using DEGMC and DDPM (sample at iteration 30).}
    \label{fig:samplemnist30}
\end{figure}
%------------------------------------------------------------%
%----------------figure des sample mnist-------------------%
\begin{figure}[!ht]
    %\vskip 0.2in
    \centering
    \begin{subfigure}{0.48\columnwidth}
        \centering
        \includegraphics[width=\linewidth]{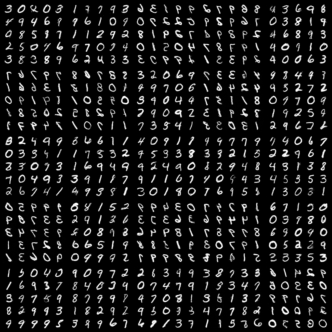}
        \caption{DEGMC (FID = 42.20)}
        \label{fig:ed2rm50}
    \end{subfigure}
    \hfill
    \begin{subfigure}{0.48\columnwidth}
        \centering
        \includegraphics[width=\linewidth]{sample-30.png}
        \caption{DDPM (FID = 38.86)}
        \label{fig:ddpm50}
    \end{subfigure}
    \caption{Generated image samples on MNIST using DEGMC and DDPM (sample at iteration 50).}
    \label{fig:samplemnist50}
\end{figure}
%------------------------------------------------------------%
%----------------figure des sample roto mnist-------------------%
\begin{figure}[!ht]
    %\vskip 0.2in
    \centering
    \begin{subfigure}{0.48\columnwidth}
        \centering
        \includegraphics[width=\linewidth]{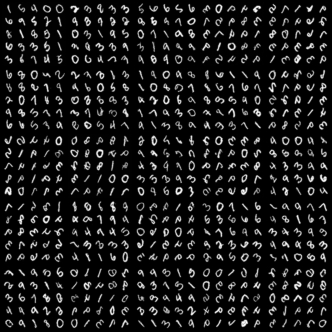}
        \caption{DEGMC (FID = 45.07)}
        \label{fig:ed2rmroto30}
    \end{subfigure}
    \hfill
    \begin{subfigure}{0.48\columnwidth}
        \centering
        \includegraphics[width=\linewidth]{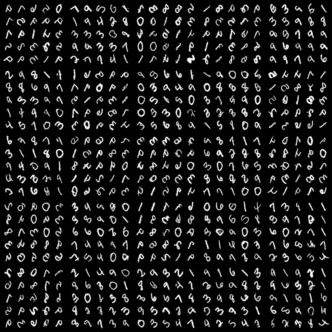}
        \caption{DDPM (FID = 48.89)}
        \label{fig:ddpmroto30}
    \end{subfigure}
    \caption{Generated image samples on RotoMNIST using DEGMC and DDPM (sample at iteration 30).}
    \label{fig:sampleroto30}
\end{figure}
%------------------------------------------------------------%
%----------------figure des sample roto mnist-------------------%
\begin{figure}[!ht]
    %\vskip 0.2in
    \centering
    \begin{subfigure}{0.48\columnwidth}
        \centering
        \includegraphics[width=\linewidth]{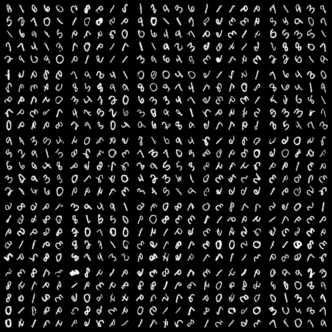}
        \caption{DEGMC (FID = 43.79)}
        \label{fig:ed2rmroto50}
    \end{subfigure}
    \hfill
    \begin{subfigure}{0.48\columnwidth}
        \centering
        \includegraphics[width=\linewidth]{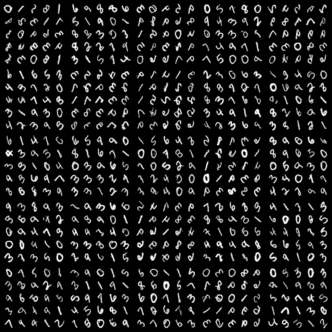}
        \caption{DDPM (FID = 47.64)}
        \label{fig:ddpmroto50}
    \end{subfigure}
    \caption{Generated image samples on RotoMNIST using DEGMC and DDPM (sample at iteration 50).}
    \label{fig:sampleroto50}
\end{figure}
%------------------------------------------------------------%

%----------------figure des sample Cifar-------------------%
\begin{figure}[!ht]
    %\vskip 0.2in
    \centering
    \begin{subfigure}{0.48\columnwidth}
        \centering
        \includegraphics[width=\linewidth]{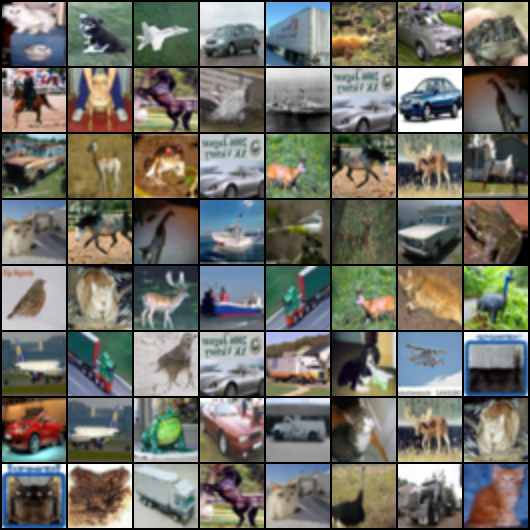}
        \caption{DEGMC (FID = 45.00)}
        \label{fig:ed2rmcifar40}
    \end{subfigure}
    \hfill
    \begin{subfigure}{0.48\columnwidth}
        \centering
        \includegraphics[width=\linewidth]{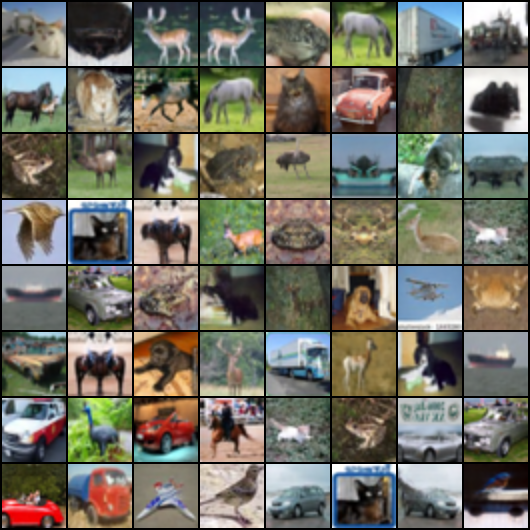}
        \caption{DDPM (FID = 45.11)}
        \label{fig:ddpmcifar40}
    \end{subfigure}
    \caption{Generated image samples on CIFAR-10 using DEGMC and DDPM (sample at iteration 40).}
    \label{fig:samplecifar40}
\end{figure}
%------------------------------------------------------------%

%----------------figure des sample cifar-------------------%
\begin{figure}[!ht]
    %\vskip 0.2in
    \centering
    \begin{subfigure}{0.48\columnwidth}
        \centering
        \includegraphics[width=\linewidth]{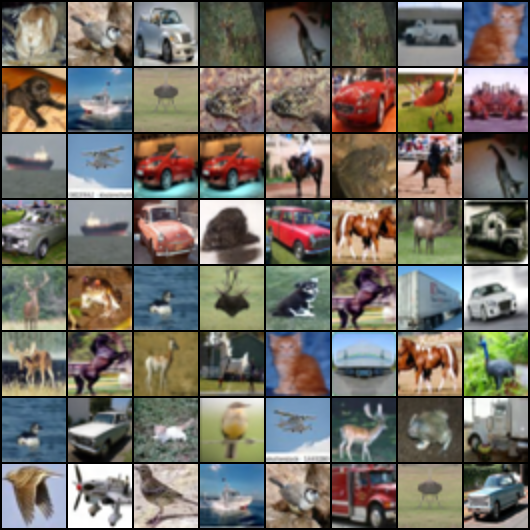}
        \caption{DEGMC (FID = 34.57)}
        \label{fig:ed2rmcifar70}
    \end{subfigure}
    \hfill
    \begin{subfigure}{0.48\columnwidth}
        \centering
        \includegraphics[width=\linewidth]{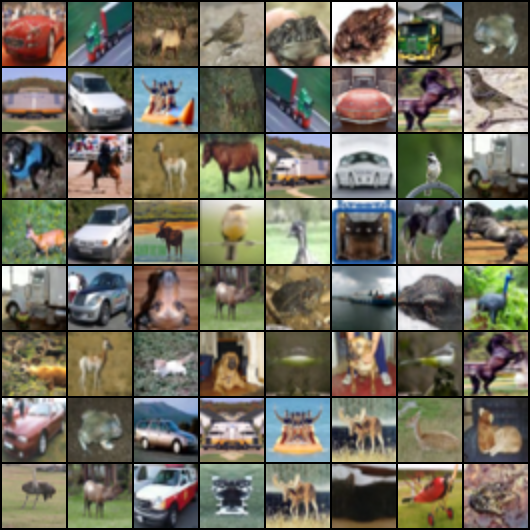}
        \caption{DDPM (FID = 34.36)}
        \label{fig:ddpmcifar70}
    \end{subfigure}
    \caption{Generated image samples on CIFAR-10 using DEGMC and DDPM (sample at iteration 70).}
    \label{fig:samplecifar70}
\end{figure}
%------------------------------------------------------------%

%\end{document}

\section{Short review on probabilistic diffusion models}\label{sec:pdm}%-----------------------------
%------------------------------------------------------------------------------

Diffusion models are generative models based on the progressive addition of noise to data, followed by learning the reverse denoising process using a neural network.

%\begin{figure}[h]
%\centering\includegraphics[width=0.7\linewidth]{fig1.png}
%\caption{Illustration of the forward noising process and the reverse denoising process.}
%\label{fig:processus}
%\end{figure}

Given observations $x \sim q(x)$, the model operates on latent variables $n_0, \dots, n_T$ of the same dimension as $x$, where $n_0$ corresponds to the observation $x$, and $n_T$ represents standard Gaussian noise (see Fig.~\ref{fig:processus}).

\paragraph{Forward process}
The forward process gradually adds noise to the variables $n_t$, for $t \in \{0, \dots, T\}$, i.e., $n_{t+1} = n_t + \text{noise}$, where the noise is random. There exists a conditional distribution that models the probability of obtaining $n_{t+1}$ given $n_t$, denoted $q(n_{t+1} \mid n_t)$, which follows the Gaussian distribution:
\begin{equation}
    q(n_t \mid n_{t-1}) = \mathcal{N}(n_t : \sqrt{\alpha_t}n_{t-1}, \Gamma_t I) \label{eq:eq10}
\end{equation}
where we define $\alpha_t$ by considering the special case of the variance-preserving process proposed by \cite{ho2020denoising}, i.e., $\alpha_t = 1 - \Gamma_t$. This parameter controls the amount of signal preserved, and $\Gamma_t \in (0,1)$ is the noise-level parameter added at each step $t$. It acts progressively in this process so that the mean $\alpha_t n_{t-1}$ increasingly deviates from the already noised data $n_{t-1}$.

The forward process is Markovian; thus, for any $t \in \{0, \dots, T\}$, $n_t$ depends only on $n_{t-1}$ and not on other variables \cite{Ghojogh2019}:
\begin{equation}
     q(n_t \mid n_{t-1}, n_{t-2}, \dots, n_0) = q(n_t \mid n_{t-1}) \label{eq:eq11}
\end{equation}
Hence, the joint distribution of this process is obtained as:
\begin{equation}
    q(n_1, n_2, \dots, n_T \mid n_0) = \prod_{t = 1}^{T} q(n_t \mid n_{t-1})
\end{equation}
Moreover, for any $t > s$, the transition distribution from $s$ to $t$ can be defined using the Gaussian reparameterization of \cite{Kingma2013}, considering the standard distribution $\varepsilon \sim \mathcal{N}(0, I)$. Thus, Eq.~\eqref{eq:eq10} can be rewritten as:
\begin{equation}
    n_t = \sqrt{\alpha_t}n_{t-1} + \sqrt{1 - \alpha_t} \, \varepsilon \label{eq:reparemetrisation}
\end{equation}
Consequently, for any $t > s$, we have:
\begin{equation}
    q(n_t \mid n_s) = \mathcal{N}(n_t : \sqrt{\alpha_{t/s}}n_s, \Gamma_{t/s} I)
\end{equation}
with $\alpha_{t/s} = \prod_{i = s+1}^{t} \alpha_i$, $\Gamma_{t/s} = 1 - \alpha_{t/s}$. With respect to the initial data $n_0$:
\begin{equation}
    q(n_t \mid n_0) = \mathcal{N}(n_t : \sqrt{\bar{\alpha}}n_0, \bar{\Gamma} I)
\end{equation}
where $\bar{\alpha} = \prod_{i = 1}^{t} \alpha_i$ and $\bar{\Gamma} = 1 - \bar{\alpha}$.

\paragraph{Reverse generative process} 
The diffusion process, or reverse generative process, progressively generates data from noise, following the true denoising process denoted $P(n_{t-1} \mid n_t)$, which defines the probability of obtaining $n_{t-1}$ from $n_t$. This distribution is Gaussian, similar to that of the forward process.

However, since $x_0$ is unknown during denoising, a neural network $\phi$, parameterized by $\theta$, is used to approximate the Gaussian reverse conditional distribution, denoted $p_{\theta}(n_{t-1} \mid n_t)$, defined as:
\begin{equation}
    p_{\theta}(n_{t-1} \mid n_t) = \mathcal{N}(n_{t-1} : \mu_{\theta}(n_t, t), \Sigma_{\theta}(n_t, t)) \label{eq:dgeneratif}
\end{equation}
where $\mu_{\theta} \in \R^d$ and $\Sigma_{\theta} \in \R^{d \times d}$ represent, respectively, the mean and covariance matrix of the distribution at iteration $t$. For simplicity, as proposed in \cite{ho2020denoising}, we fix $\Sigma_{\theta}(n_t, t) = \sigma_{t}^{2} I$, with time-dependent constants $\sigma_t^2$ that are not learned.

%\begin{figure}[h]
%\centering\includegraphics[width=0.9\linewidth]{fig 2.png}
%\caption{Conditional distributions in the forward and generative processes of the diffusion model.}
%\label{fig:distribution}
%\end{figure}

The generative process is also modeled as a first-order Markov chain, i.e.,
\begin{equation}
    P(n_{t-1} \mid n_t, n_{t+1}, \dots, n_T) = P(n_{t-1} \mid n_t)
\end{equation}
so that the joint distribution of the generative process is written as:
\begin{equation}
    p_{\theta}(n_{0:T}) = P(n_T) \prod_{t=1}^{T} p_{\theta}(n_{t-1} \mid n_t)
\end{equation}
with $P(n_T)$ generally defined as Gaussian noise:
\begin{equation}
    P(n_T) = \mathcal{N}(0, I)
\end{equation}

The true distribution is defined similarly to Eq.~\eqref{eq:eq10}:
\begin{equation}
    P(n_{t-1} \mid n_t, n_0) = \mathcal{N}(n_{t-1} : \tilde{\mu}(n_t, n_0), \tilde{\sigma}^2 I) \label{eq:dreel}
\end{equation}
Using the same reparameterization technique as in the forward case (see Eq.~\eqref{eq:reparemetrisation}), one can sample from a standard distribution $\varepsilon \sim \mathcal{N}(0, I)$. It can then be shown that the mean of Eq.~\eqref{eq:dreel} can be expressed as:
\begin{equation}
    \tilde{\mu}(n_t, n_0) = \frac{1}{\sqrt{\alpha_t}} \left(n_t - \frac{\Gamma_t}{\sqrt{1 - \bar{\alpha}_t}} \varepsilon \right)
\end{equation}
Thus, an expression for $\mu_{\theta}(n_t, t)$ in Eq.~\eqref{eq:dgeneratif} can also be given. Since the learned denoising process is defined based on the true denoising process, we obtain:
\begin{equation}
    \mu_{\theta}(n_t, t) = \frac{1}{\sqrt{\alpha_t}} \left(n_t - \frac{\Gamma_t}{\sqrt{1 - \bar{\alpha}_t}} \varepsilon_{\theta}(n_t, t) \right)
\end{equation}
where $\varepsilon_{\theta}(n_t, t) = \phi(n_t, t)$ is the output of the neural network $\phi$ at iteration $t$.

\paragraph{Variational lower bound of the likelihood} 
As previously mentioned, diffusion models introduce a sequence of latent variables. The likelihood of the data in these models is expressed as follows:
\begin{align*}
 p_{\theta}\left(\boldsymbol{n}_{0}\right)=\int p_{\theta}\left(\boldsymbol{n}_{0}, \boldsymbol{n}_{1: T}\right) \, d \boldsymbol{n}_{1: T}. 
\end{align*}

Since the direct maximization of this likelihood is intractable, diffusion models optimize a variational lower bound (ELBO) of the data likelihood:
\begin{align}
\mathcal{L} := \; &\E_{q\left(n_{1: T} \mid     n_{0}\right)}\!\left[\log p_{\theta}\left(n_{0} \mid n_{1: T}\right)\right] - \operatorname{KL}\!\left(q\left(n_{1: T} \mid n_{0}\right) \,\|\, p_{\theta}\left(n_{1: T}\right)\right) 
\;\; \leq \;\; \log p_{\theta}\left(n_{0}\right),
\end{align}
where $\E[\cdot]$ denotes the expectation and $\operatorname{KL}(\cdot \| \cdot)$ the Kullback–Leibler divergence.  

By further expansion, we obtain an equivalent expression (see \cite{ho2020denoising,SohlDickstein2015}):
\begin{align*}
\mathcal{L}(\theta) = -\operatorname{KL}\!\left(q\left(n_{T} \mid n_{0}\right) \,\|\, p\left(n_{T}\right)\right) 
 - \sum_{t=2}^{T} \operatorname{KL}\!\left(q\left(n_{t-1} \mid n_{t}, n_{0}\right) \,\|\, p_{\theta}\left(n_{t-1} \mid n_{t}\right)\right) \\
+ \E_{q\left(n_{1:T} \mid n_{0}\right)}\!\left[\log p_{\theta}\left(n_{0} \mid n_{1}\right)\right].
\end{align*}

The variational lower bound must be maximized with respect to $\theta$. However, the first KL divergence is independent of $\theta$ and can therefore be ignored during optimization. Thus, maximizing $\mathcal{L}$ is equivalent to solving:
\begin{align}
\underset{\theta}{\operatorname{minimize}} \;  \sum_{t=2}^{T} \operatorname{KL}\!\left(q\left(n_{t-1} \mid n_{t}, n_{0}\right) \,\|\, p_{\theta}\left(n_{t-1} \mid n_{t}\right)\right) 
 - \E_{q\left(n_{1: T} \mid n_{0}\right)}\!\left[\log p_{\theta}\left(n_{0} \mid n_{1}\right)\right].
\end{align}

This equation drives the reverse distribution $p_{\theta}\!\left(\boldsymbol{x}_{t-1} \mid \boldsymbol{x}_{t}\right)$ to approximate the true denoising distribution $q\!\left(\boldsymbol{x}_{t-1} \mid \boldsymbol{x}_{t}, \boldsymbol{x}_{0}\right)$ by minimizing their KL divergence. It can therefore be used as a loss function for a neural network parametrized by $\theta$, emphasizing the alignment between these two distributions.  

In other words, optimizing the ELBO forces the model to learn a denoising process capable of reversing the progressive diffusion of noise. Hence, training consists in aligning the learned reverse process $p_{\theta}$ with the true denoising process $q$, while maximizing the likelihood of the real data.

%--------------------------------------------------------
\bibliographystyle{IEEEtran}
\bibliography{Le_Thier_bib,biblio} % ../../../BIBLIOGRAPHY/bib/

\begin{thebibliography}{10}
\providecommand{\url}[1]{#1}
\csname url@rmstyle\endcsname
\providecommand{\newblock}{\relax}
\providecommand{\bibinfo}[2]{#2}
\providecommand\BIBentrySTDinterwordspacing{\spaceskip=0pt\relax}
\providecommand\BIBentryALTinterwordstretchfactor{4}
\providecommand\BIBentryALTinterwordspacing{\spaceskip=\fontdimen2\font plus
\BIBentryALTinterwordstretchfactor\fontdimen3\font minus
  \fontdimen4\font\relax}
\providecommand\BIBforeignlanguage[2]{{%
\expandafter\ifx\csname l@#1\endcsname\relax
\typeout{** WARNING: IEEEtran.bst: No hyphenation pattern has been}%
\typeout{** loaded for the language `#1'. Using the pattern for}%
\typeout{** the default language instead.}%
\else
\language=\csname l@#1\endcsname
\fi
#2}}

\bibitem{Goodfellow2014}
I.~Goodfellow, J.~Pouget-Abadie, M.~Mirza, B.~Xu, D.~Warde-Farley, S.~Ozair,
  A.~Courville, and Y.~Bengio, ``Generative adversarial nets,'' \emph{Advances
  in neural information processing systems}, vol.~27, 2014.

\bibitem{Kingma2013}
D.~P. Kingma and M.~Welling, ``Auto-encoding variational bayes,'' \emph{arXiv
  preprint arXiv:1312.6114}, 2013.

\bibitem{Kingma2014}
D.~P. Kingma, S.~Mohamed, D.~Jimenez~Rezende, and M.~Welling, ``Semi-supervised
  learning with deep generative models,'' \emph{Advances in neural information
  processing systems}, vol.~27, 2014.

\bibitem{dhariwal2021diffusion}
P.~Dhariwal and A.~Nichol, ``Diffusion models beat gans on image synthesis,''
  \emph{Advances in neural information processing systems}, vol.~34, pp.
  8780--8794, 2021.

\bibitem{ho2020denoising}
J.~Ho, A.~Jain, and P.~Abbeel, ``Denoising diffusion probabilistic models,''
  \emph{Advances in neural information processing systems}, vol.~33, pp.
  6840--6851, 2020.

\bibitem{chen2020wavegrad}
N.~Chen, Y.~Zhang, H.~Zen, R.~J. Weiss, M.~Norouzi, and W.~Chan, ``Wavegrad:
  Estimating gradients for waveform generation,'' \emph{arXiv preprint
  arXiv:2009.00713}, 2020.

\bibitem{popov2021grad}
V.~Popov, I.~Vovk, V.~Gogoryan, T.~Sadekova, and M.~Kudinov, ``Grad-tts: A
  diffusion probabilistic model for text-to-speech,'' in \emph{International
  conference on machine learning}.\hskip 1em plus 0.5em minus 0.4em\relax PMLR,
  2021, pp. 8599--8608.

\bibitem{simonovsky2018graphvae}
M.~Simonovsky and N.~Komodakis, ``Graphvae: Towards generation of small graphs
  using variational autoencoders,'' in \emph{International conference on
  artificial neural networks}.\hskip 1em plus 0.5em minus 0.4em\relax Springer,
  2018, pp. 412--422.

\bibitem{gebauer2019symmetry}
N.~Gebauer, M.~Gastegger, and K.~Sch{\"u}tt, ``Symmetry-adapted generation of
  3d point sets for the targeted discovery of molecules,'' \emph{Advances in
  neural information processing systems}, vol.~32, 2019.

\bibitem{simm2020symmetry}
G.~N. Simm, R.~Pinsler, G.~Cs{\'a}nyi, and J.~M. Hern{\'a}ndez-Lobato,
  ``Symmetry-aware actor-critic for 3d molecular design,'' \emph{arXiv preprint
  arXiv:2011.12747}, 2020.

\bibitem{hoogeboom2022equivariant}
E.~Hoogeboom, V.~G. Satorras, C.~Vignac, and M.~Welling, ``Equivariant
  diffusion for molecule generation in 3d,'' in \emph{International conference
  on machine learning}.\hskip 1em plus 0.5em minus 0.4em\relax PMLR, 2022, pp.
  8867--8887.

\bibitem{SohlDickstein2015}
J.~Sohl-Dickstein, E.~A. Weiss, N.~Maheswaranathan, and S.~Ganguli, ``Deep
  unsupervised learning using nonequilibrium thermodynamics,'' in
  \emph{ICML'15: Proceedings of the 32nd International Conference on
  International Conference on Machine Learning}, vol.~37, Lille, France, July
  2015, pp. 2256 -- 2265.

\bibitem{Song2019}
Y.~Song and S.~Ermon, ``Generative modeling by estimating gradients of the data
  distribution,'' in \emph{Proceedings of the 33rd International Conference on
  Neural Information Processing Systems}, Vancouver, Canada, Dec. 2019, pp.
  11\,918 -- 1193.

\bibitem{Song2021}
Y.~Song, J.~Sohl-Dickstein, D.~P. Kingma, A.~Kumar, S.~Ermon, and B.~Poole,
  ``Score-{B}ased {G}enerative {M}odeling through {S}tochastic {D}ifferential
  {E}quations,'' in \emph{International Conference on Learning
  Representations}, Vienna, Austria, May 2021.

\bibitem{Croitoru2023}
F.-A. Croitoru, V.~Hondru, R.~T. Ionescu, and M.~Shah, ``Diffusion {M}odels in
  {V}ision: A {S}urvey,'' \emph{IEEE Transactions on Pattern Analysis and
  Machine Intelligence}, vol.~45, no.~9, pp. 10\,850--10\,869, Sept. 2023.

\bibitem{Huang2021}
C.-W. Huang, J.~H. Lim, and A.~Courville, ``A {V}ariational {P}erspective on
  {D}iffusion-{B}ased {G}enerative {M}odels and {S}core {M}atching,'' in
  \emph{Conference on Neural Information Processing Systems (NeurIPS 2021)},
  San Diego, CA, Dec. 2021.

\bibitem{Rezende2014}
D.~J. Rezende, S.~Mohamed, and D.~Wierstra, ``Stochastic {B}ackpropagation and
  {A}pproximate {I}nferencein {D}eep {G}enerative {M}odels,'' in
  \emph{Proceedings of the International Conference on MachineLearning},
  Beijing, China, 2014.

\bibitem{Goodfellow2017}
I.~Goodfellow, ``{G}enerative {A}dversarial {N}etworks,'' in \emph{NIPS}, 2017,
  p.~57.

\bibitem{Satorras2021}
V.~G. Satorras, E.~Hoogeboom, and M.~Welling, ``E(n) {E}quivariant {G}raph
  {N}eural {N}etworks,'' in \emph{International Conference on
  MachineLearning}.\hskip 1em plus 0.5em minus 0.4em\relax PMLR, 2021, pp.
  9323--9332.

\bibitem{Cornet2024}
F.~Cornet, G.~Bartosh, M.~N. Schmidt, and C.~A. Naesseth, ``Equivariant
  {N}eural {D}iffusion for {M}olecule {G}eneration,'' in \emph{Conference on
  Neural Information Processing Systems}, 2024.

\bibitem{Brehmer2023}
J.~Brehmer, J.~Bose, P.~de~Haan, and T.~Cohen, ``Edgi: Equivariant {D}iffusion
  for {P}lanning with {E}mbodied {A}gents,'' in \emph{Conference on Neural
  Information Processing Systems}, 2023.

\bibitem{Igashov2024}
I.~Igashov, H.~Stärk, C.~Vignac, A.~Schneuing, V.~G. Satorras, P.~Frossard,
  M.~Welling, M.~Bronstein, and B.~Correia, ``Equivariant 3{D}-conditional
  diffusion model for molecular linker design,'' \emph{Nature Machine
  Intelligence}, vol.~6, no.~4, pp. 417--427, Apr. 2024.

\bibitem{Cohen2016}
T.~Cohen and M.~Welling, ``{G}roup {E}quivariant {C}onvolutional {N}etworks,''
  in \emph{International conference on machine learning}.\hskip 1em plus 0.5em
  minus 0.4em\relax PMLR, 2016, pp. 2990--2999.

\bibitem{Bekkers2018}
E.~J. Bekkers, M.~W. Lafarge, M.~Veta, K.~A. Eppenhof, J.~P. Pluim, and
  R.~Duits, ``{R}oto-translation covariant convolutional networks for medical
  image analysis,'' in \emph{Medical Image Computing and Computer Assisted
  Intervention -- MICCAI 2018: 21st International Conference, Proceedings, Part
  I}, Granada, Spain, Sept. 2018, pp. 440--448.

\bibitem{Cohen2019}
T.~S. Cohen, M.~Geiger, and M.~Weiler, ``A general theory of equivariant cnns
  on homogeneous spaces,'' \emph{Advances in neural information processing
  systems}, vol.~32, 2019.

\bibitem{Winkels2018}
M.~Winkels and T.~S. Cohen, ``3d {G}-{CNN}s for pulmonary nodule detection,''
  in \emph{Medical Imaging with Deep Learning}, 2018.

\bibitem{Cohen2018}
T.~S. Cohen, M.~Geiger, J.~Köhler, and M.~Welling, ``Spherical {CNN}s,'' in
  \emph{International Conference on Learning Representations}, 2018.

\bibitem{Bekkers2019}
E.~Bekkers, ``B-{S}pline {CNN}s on {L}ie {G}roups,'' in \emph{International
  Conference on Learning Representations}, 2019.

\bibitem{Smets13July2022}
B.~M.~N. Smets, J.~Portegies, E.~J. Bekkers, and R.~Duits, ``{PDE}-{B}ased
  {G}roup {E}quivariant {C}onvolutional {N}eural {N}etworks,'' \emph{Journal of
  Mathematical Imaging and Vision}, vol.~65, no.~1, pp. 209--239, 2022.

\bibitem{Bellaard2023}
G.~Bellaard, D.~L. Bon, G.~Pai, B.~M. Smets, and R.~Duits, ``{A}nalysis of
  (sub-){R}iemannian {PDE-G-CNNs},'' \emph{Journal of Mathematical Imaging and
  Vision}, pp. 1--25, 2023.

\bibitem{Diop2024b}
E.~H.~S. Diop, T.~Fall, A.~Mbengue, and M.~Daoudi, \emph{{GM}-{GAN}: Geometric
  {G}enerative {M}odels {B}ased on {M}orphological {E}quivariant {PDE}s and
  {GAN}s}.\hskip 1em plus 0.5em minus 0.4em\relax Springer Nature Switzerland,
  Dec. 2024, pp. 310--325.

\bibitem{Huang2022}
C.-W. Huang, M.~Aghajohari, A.~J.~B. Panangaden, and A.~Courville, ``Riemannian
  {D}iffusion {M}odels,'' in \emph{Conference on Neural Information Processing
  Systems Conference on Neural Information Processing Systems}, 2022.

\bibitem{Liu2025}
Z.~Liu, W.~Zhang, C.~Schütte, and T.~Li, ``Riemannian {D}enoising {D}iffusion
  {P}robabilistic {M}odels,'' May 2025.

\bibitem{Bortoli2022}
V.~D. Bortoli, Émile Mathieu, M.~Hutchinson, J.~Thornton, Y.~W. Teh, and
  A.~Doucet, ``Riemannian score-based generative modelling,'' \emph{Advances in
  neural information processing systems}, vol.~35, pp. 2406--2422, 2022.

\bibitem{Lou2023}
A.~Lou, M.~Xu, A.~Farris, and S.~Ermon, ``Scaling {R}iemannian {D}iffusion
  {M}odels,'' in \emph{Conference on Neural Information Processing Systems},
  2023.

\bibitem{Fathi2008}
A.~Fathi, \emph{The {W}eak {KAM} {T}heorem in {L}agrangian {D}ynamics}.\hskip
  1em plus 0.5em minus 0.4em\relax Cambridge University Press, 2008.

\bibitem{Diop2021}
E.~H.~S. Diop, A.~Mbengue, B.~Manga, and D.~Seck, ``Extension of {M}athematical
  {M}orphology in {R}iemannian {S}paces,'' in \emph{Lecture Notes in Computer
  Science}.\hskip 1em plus 0.5em minus 0.4em\relax Springer International
  Publishing, 2021, pp. 100--111.

\bibitem{Ghojogh2019}
B.~Ghojogh, F.~Karray, and M.~Crowley, ``Hidden {M}arkov {M}odel: Tutorial,''
  \emph{engrXiv}, July 2019.

\bibitem{Meyer2000}
F.~Meyer and P.~Maragos, ``Nonlinear scale-space representation with
  morphological levelings,'' \emph{Journal of Visual Communication and Image
  Representation}, vol.~11, pp. 245--265, 2000.

\bibitem{Schmidt2016}
M.~Schmidt and J.~Weickert, ``Morphological counterparts of linear
  shift-invariant scale-spaces,'' \emph{Journal of Mathematical Imaging and
  Vision}, vol.~56, no.~2, pp. 352--366, apr 2016.

\bibitem{Barron2021}
E.~N. Barron, ``A survey of {H}opf-{L}ax {F}ormulas and {Q}uasiconvexity in
  {PDE}s,'' \emph{Trudy Instituta Matematiki i Mekhaniki UrO RAN}, vol.~27,
  no.~3, pp. 237--245, Sept. 2021.

\bibitem{Donato2023}
D.~D. Donato, ``The intrinsic hopf-lax semigroup vs. the intrinsic slope,''
  \emph{Journal of Mathematical Analysis and Applications}, vol. 523, no.~2, p.
  127051, jul 2023.

\end{thebibliography}

\end{document}